\documentclass[twoside,11pt]{article}

\usepackage{jmlr2e}
\usepackage{amsmath}
\usepackage{comment}
\usepackage{xcolor}
\usepackage{bbm}
\usepackage{enumerate} 

\DeclareMathOperator*{\argmax}{argmax}
\def\bs{\boldsymbol}
\def\P{\mathbb P}
\def\R{\mathbb R}
\def\1{\mathbf{1}}
\def\E{\mathbb{E}}
\def\e{\epsilon}

\def\a{\alpha}

\def\bS{\mathbb{S}}
\def\cA{\mathcal{A}}

\def\pa{\partial}

\jmlrheading{1}{2000}{1-48}{4/00}{10/00}{meila00a}{Erhan Bayraktar, Ibrahim Ekren and Xin Zhang}

\ShortHeadings{Prediction against a limited adversary}{Bayraktar, Ekren and Zhang}
\firstpageno{1}

\begin{document}

\title{Prediction against a limited adversary}

\author{\name Erhan Bayraktar \email erhan@umich.edu \\
       \addr Department of Mathematics\\
       University of Michigan\\
       Ann Arbor, MI 48109-1043, USA
       \AND
       \name Ibrahim Ekren \email iekren@fsu.edu \\
       \addr Department of Mathematics\\
       Florida State Univeristy\\
       Tallahassee, FL 32306-4510, USA
       \AND
       \name Xin Zhang \email zxmars@umich.edu \\
       \addr Department of Mathematics\\
       University of Michigan\\
       Ann Arbor, MI 48109-1043, USA}

\editor{}

\maketitle

\begin{abstract}
We study the problem of prediction with expert advice with adversarial corruption where the adversary can at most corrupt one expert. Using tools from viscosity theory, we characterize the long-time behavior of the value function of the game between the forecaster and the adversary. We provide lower and upper bounds for the growth rate of regret without relying on a comparison result. We show that depending on the description of regret, the limiting behavior of the game can significantly differ. 
\end{abstract}


\section{Introduction}

Prediction with expert advice is one of the fundamental problems in \emph{online learning} and \emph{sequential decision making}. In this problem, at each round a forecaster chooses between alternative actions based on his current and past observations with the objective of performing as well as the best constant strategy. We refer the reader to \cite{cesa2006prediction} for a survey. This problem is often studied in the adversarial setting where an adversary chooses the outcomes to maximize the regret of the forecaster. This interaction between the forecaster and the adversary can be seen as a zero-sum game (see e.g. \cite{NIPS2017_6896,DBLP:conf/colt/AbernethyABR09,MR4120922,2019arXiv190202368B,2020arXiv200800052D,MR4053484,10.5555/2884435.2884474,NIPS2012_4638}).
Using the minimax theorem, one can easily show that this zero-sum game admits a value under mild assumptions and the value function satisfies a discrete time dynamic programming principle. Then, the long-time behavior of the value function can be studied by showing that the discrete time dynamic programming equation ``converges" to a differential operator and a scaled version of the value function converges to the solution of a partial differential equation associated to the differential operator. Viscosity solution theory provides formidable tools to rigorously show this convergence and study the properties of the long-time behavior of the value function. 

One can also state the prediction problem in the stochastic setting where the actions of the adversary are drawn from a fixed distribution (unknown or known to the forecaster). Since the decisions of the adversary do not depend on the state, the forecaster has better performances and his regret is smaller. 

Similar to \cite{2020arXiv200210286A,2020arXiv200308457B,pmlr-v99-gupta19a,MR3918712,10.1145/3188745.3188918}, in this paper, we bridge the adversarial and stochastic settings by considering an adversary who cannot freely choose the outcomes. In our framework, the gains of the experts are drawn from a fixed distribution. Then, without seeing the outcomes and each other's decisions, the adversary chooses to corrupt the gain of one of the experts and the forecaster chooses one of the experts. If the forecaster chose the corrupted expert,  he obtains the corrupted gain. Otherwise he obtains the gain of the expert he chose. 
By studying the value function of this game between the adversary and the forecaster, we show that several important features of the fully adversarial setting do not extend to our framework and the assumptions on the data of the problem can lead to dramatic differences for the long-time behavior of regret.

First of all, we show that the existence of the value for the zero-sum game in the pre-limit regime is not guaranteed. Indeed, if one does not state the problem of the adversary properly, the strategies of the adversary might fail to range in a convex set. This point has crucial implications. Indeed, the minimax theorem fails and one cannot establish a dynamic programming equation and the analysis of the interaction becomes significantly more challenging. In our work, we identify a relevant set of strategies for the adversary that allows us to obtain the existence of the value and to use the viscosity machinery.

The second contribution of our paper is to exhibit wildly different behavior of the regret in the long-time regime for different types of final conditions for the zero-sum game. In the classical statement of the prediction problem the gain of the forecaster is compared against the gain of the best expert. In this case, the payoff function at maturity of the zero-sum game is given by $\Phi_m(x):=\max_i x^i$, $\forall x\in \R^N$, where $N$ is the number of experts in this game. One fundamental question is whether the long-time behavior of the prediction problem is robust with respect to the choice of this payoff function. Different choices of payoff functions are made in \cite[Proposition 4.1]{hart2001general}, \cite{2020arXiv200800052D,MR4053484}, see also the distinction between internal and external regret in \cite{foster1999regret}. In particular, in \cite{2020arXiv200800052D}, the authors assume that the payoff function satisfies a strict monotonicity condition which is for example not satisfied by the function $\Phi_m$.  Since the choice of the payoff function only impacts the final condition of the associated partial differential equation, the viscosity solution approach is a formidable tool to study the impact of the payoff function on the growth of the regret. Using these tools, we show that the long-time behavior of the regret have different regimes depending on whether we assume this strict monotonicity. 

The third contribution of our paper is to show that, although mathematically appealing, a comparison result for viscosity solutions of the limiting equation is not fundamental to obtain algorithms for the forecaster and the adversary and the growth of the regret. Indeed, similarly to {\cite{kobzar2020new,pmlr-v125-kobzar20a}}, algorithms for the adversary and a lower bound for the growth of regret can be found using a smooth subsolution of the limiting equation. Additionally, by considering a smooth supersolution of a relevant equation, one can construct an algorithm for the forecaster and an upper bound for the growth of the regret. As in Theorem~\ref{thm:viscosity2}, usually one can show that the infimum (supremum) limit of scaled value functions is a supersolution (subsolution) of the limiting equation. Therefore if a comparison result for viscosity solutions exists, one can conclude that the scaled value function converges and thus obtain the exact growth rate of regret. Note also that the Hamiltonian of the limiting equation we obtain has a discontinuous dependence on the first derivative and the equation is similar to the geometric equations studied in \cite{MR1100211,MR1119185,MR1924400,MR1988466}.

Finally, unlike in \cite{MR4120922,2020arXiv200813703C,MR4053484} where the gradient (or simple transformation of the gradient) of the solution to the limiting equation yields an asymptotic optimal algorithm for the forecaster, we show that this gradient may not provide an asymptotic optimal algorithm for the forecaster in our problem. Unfortunately, this point shows that the solution to the limiting equation might fail to capture some feature of the prediction problem. There are other variants of the prediction problem, e.g., \cite{10.5555/3044805.3044832} considers online learning when the time horizon is unknown and similar to our case the controls of the adversary are limited, and \cite{NIPS2010_0a0a0c8a} studies a repeated zero-sum game where an adversary plays on a budget.

The rest of the paper is organized as follows. In Section \ref{s.pf}, we formulate the problem of prediction against a limited adversary and state the relevant assumptions. In Section \ref{sec:PDE}, we heuristically derive the limiting equation and in Section \ref{sec:main} state our main results. The Section \ref{sec:special} contains special cases where we can explicitly solve the limiting equation.

\subsection{Notations}
Let $N\geq 2$ and denote $\{e^i\}_{i=1}^N$ the canonical basis of $\R^N$. We define $\1=\sum_{i=1}^N e^i$, $\R_+^N=[0,\infty)^N$. We denote by $\bS_N$ the set of symmetric matrices of dimension $N$. 
\
\section{Problem Formulation}\label{s.pf}
Consider a learning system with $N\geq 2$ experts, an adversary and a forecaster.  At each round $m$, each expert $i\in \{1,\ldots, N\}$ makes a prediction which yields a gain $g^i_m\in \{0,1\}$. Here $g^i_m=0$ (resp. $g^i_m=1$) represents that the prediction is wrong (resp. correct) at this round. We assume that each expert is correct with probability $\mu^i$, i.e., $\E[g^i_m]=\mu^i\in[0,1]$. Knowing the values of $\{\mu^i:i\in\{1,\ldots,N\}\}$, the adversary and the forcaster play a zero-sum game. At each round, the adversary picks one expert $A_m \in \{1,\ldots,N\}$, and sets his gain to $h_m \in \{0,1\}$. The adversary uses mixed type strategies and therefore, he chooses a distribution for $(A_m, h_m) \in \{1,\dotso,N \} \times \{0,1\}$ that may depend on the past history of the game. 
The realized gain $\Delta G^i_m$ of the expert $i$ at round $m$ is 
$$\Delta G^i_m=g^i_m\1_{\{i\neq A_m\}}+h_m\1_{\{i=A_m\}}$$
and the total gain of the expert $i$ is
$$G^i_m=\sum_{k=1}^m\Delta G^i_k.$$
The fact that the adversary can only interfere on the outcome of the prediction of one expert is the main difference between our framework and the classical prediction with expert advice problems in \cite{DBLP:conf/colt/AbernethyABR09,cesa1997use,cesa2006prediction,MR4053484,10.5555/2884435.2884474}, and also the bandit problems with corruption such as \cite{pmlr-v99-gupta19a} and \cite{10.1145/3188745.3188918} where the regret bounds provided depend on the corruption. However, unlike \cite{2020arXiv200210286A} and \cite{MR3918712}, the adversary can optimally control the level of corruption at each round and therefore the level of corruption might be unbounded. 

If the forecaster chooses to follow expert $F_m \in \{1,\ldots,N\}$ at each round, then his gain is given by 
$$G_m:=\sum_{k=1}^m \Delta G_k:=\sum_{k=1}^m \Delta G^{F_k}_k.$$
The state of the zero-sum game between the adversary and the forecaster is 
$$X_m=(X^1_m,\ldots, X^N_m)=(G^1_m-G_m,\ldots G^N_m-G_m)$$
which evolves as
\begin{align*}
\Delta X_m&=(\Delta G^1_m-\Delta G_m,\ldots,\Delta G^N_m-\Delta G_m).
\end{align*}
Given the state, the control of the adversary is $\alpha_m=\{(a_m^i,b^i_m)\}_{i=1,\ldots N}$ where 
$$a_m^i=\P(A_m=i, h_m=0), \quad b^i_m=\P(A_m=i, h_m=1),$$
and the control of the forecaster is $\phi_m=\{\phi_m^i\}_{i=1,\ldots N}$ where
$$\phi_m^i=\P(F_m=i).$$
We assume that the random variables $\{g^i_m\}\cup\{(A_m,h_m)\}\cup\{F_m\}$ are mutually independent. 
 \begin{remark}\label{rem:convex} 
We do not assume that $A_m$ and $h_m$ are independent and this point is crucial. Indeed, in the definition of admissible strategies, if we require $A_m$ and $h_m$ to be independent, then the set of admissible distributions of $\Delta G_m$  might fail to be convex. Then, we would not be able to apply the minimax theorem to have a saddle point for the interaction between the adversary and the forecaster. 

However, since we assume that $A_m$ and $h_m$ are not required to be independent, the set of distributions of $\Delta G_m$ is isomorphic to 
$$\cA:=\left\{((a^i)_{i=1}^N,(b^i)^{N}_{i=1})\in [0,1]^N\times [0,1]^N: \sum_{i=1}^N a^i+b^i=1\right\},$$
which is convex. 
 \end{remark}
Simple computation yields that for all $j\in\{1,\ldots, N\}$,
\begin{align}
\mathbb{E}^{\alpha_m} [\Delta G_m^j]&= (1-a_m^j-b_m^j)\mu^j+b_m^j, \label{eq:eX1} \\
 \mathbb{E}^{\phi_m, \alpha_m} [\Delta X_m^j]&=(1-a_m^j-b_m^j)\mu^j+b_m^j-\sum_{i=1}^N\phi^i_m((1-a_m^i-b_m^i)\mu^i+b_m^i). \label{eq:eX}
 \end{align}

Suppose the maturity is $M>0$ and let $\Phi:\R^N\mapsto \R$ be a given function. We define the regret of the forecaster via
\begin{align*}
\Phi(X_M)=\Phi( G_M^1-G_M,\ldots , G_M^N-G_M).
\end{align*}
We now state the following assumptions on $\Phi$. 
\begin{assume}[Assumptions on the final condition]\label{assume:final}
(i) $\Phi$ is Lipschitz continuous and increasing in the sense that 
$$\Phi(x+y)\geq \Phi(x)\mbox{ for all }x\in \R^N \mbox{ and }y\in \R_+^N.$$

(ii)For all $x\in \R^N$ and $\lambda>0$, $\Phi(\lambda x)=\lambda \Phi(x)$ and 
$\Phi(x+\lambda \1)= \Phi(x)+\lambda$. 

(iii)There exists $\theta>0$ so that $$\Phi(x+y)\geq \Phi(x)+\frac{\theta}{N} y\cdot \1 \mbox{ for all }x\in \R^N\mbox{ and }y\in \R_+^N.$$
\end{assume} 
Trivially, (i) and (ii) holds for classical examples of functions such as  
\begin{align}\label{defpm}
\Phi_m(x): =\max_i x^i.
\end{align}
However, this choice of final value does not satisfy (iii). In order to satisfy all the assumption, one can perturb the function $\Phi_m(x)$ as
$$\Phi_{m,\theta}(x):= \left(1-\theta\right)\Phi_m(x)+\frac{\theta}{N} \sum_i x^i$$ for $\theta \in (0,1)$ by making the forecaster partially satisfied if he does better than the average.
Our Theorems \ref{thm:viscosity2} and \ref{thm:viscosity1} 
below state that the leading order expansion of the regret crucially depends on whether $\Phi$ satisfies the Assumption \ref{assume:final} (iii) or not.

The objective of the forecaster is to minimize his expected regret at maturity $M$ while the objective of the adversary is to maximize the regret of the forecaster. Then, given the terminal condition  $\Phi$, for $x\in \R^N$ and $m\in \{0,\ldots,M-1\}$, we can define the value function of interest via the iteration 
\begin{align}
V^M(M,x)&:=\Phi(x) \label{def:terminal}\\
V^M(m,x)&:=\min_{\phi_m}\max_{\alpha_m} \mathbb{E}^{\phi_m, \alpha_m} [V^M(m+1,x+ \Delta X_m  ) ] \label{def:V},
\end{align}
where $\E^{\phi_m,\alpha_m}$ is the expectation given the choices of $\phi_m$ and $\alpha_m$. Since the space of strategies is the same for each $m$, we might suppress $m$ from notation $\phi_m,\alpha_m, \Delta G_m, \Delta X_m$ in the dynamic programming equation \eqref{def:V}.

We have the following result for the value function. 
\begin{lemma}\label{lem:property}
Under Assumption \ref{assume:final} (i) and (ii), for all $m\in \{0,\ldots, M\}$ and $(x,y)\in  \R^N\times \R^N_+$, we have the following relations
\begin{align}
V^M(m,x)&=\max_{\alpha}\min_{\phi} \mathbb{E}^{\phi, \alpha} [V^M(m+1,x+ \Delta X  ) ]\\
V^M(m,x+\lambda \1)&=V^M(m,x)+\lambda,\, \mbox{ and }V^M(m,x+y)\geq V^M(m,x).\label{eq:translation}
\end{align}
If we also make the Assumption \ref{assume:final} (iii), then
\begin{align}\label{eq:strict}
V^M(m,x+y)\geq V^M(m,x)+\frac{\theta}{N} y\cdot \1.
\end{align}
\end{lemma}
\begin{proof}
The integrability of the random variables are a direct consequence of the Lipschitz continuity of $\Phi$ that passes to $V^M$ by induction. 
It is clear that for all $\phi$ the mapping 
$\a\mapsto \E^{\phi,\a}[V^M(m+1,x+ \Delta X  )]$ is linear therefore convex. Similarly, for all $\a$ the mapping 
$\phi\mapsto \E^{\phi,\a}[V^M(m+1,x+ \Delta X  )]$ is concave. Given the Remark \ref{rem:convex}, we can apply the classical minimax theorem to commute the min and the max. \eqref{eq:translation} is a simple consequence of the invariance of the final condition $\Phi$ in Assumption \ref{assume:final} (ii) and similarly \eqref{eq:strict} is a consequence of Assumption \ref{assume:final} (iii). 
\end{proof}

\section{PDE describing the long-time regime}\label{sec:PDE}

In order to study the behavior of $V^M$ for large $M$, we define the scaled value function as 
\begin{align*}
u^{M}(t,x) = \frac{1}{\sqrt{M}} V^M \left(\lceil Mt \rceil, \sqrt{M}x \right).
\end{align*}
Thanks to Lemma \ref{lem:property}, it can be easily seen that $u^M$ satisfies equations 
\begin{align}\label{eq:sdpp}
u^M(t,x)&=\min_{\phi}\max_{\alpha} \mathbb{E}^{\phi,\alpha} \left[u^M\left( t+\frac{1}{M}, x+\frac{1}{\sqrt{M}} \Delta X \right) \right]  \\
&=\max_{\alpha} \min_{\phi}\mathbb{E}^{\phi,\alpha} \left[u^M\left( t+\frac{1}{M}, x+\frac{1}{\sqrt{M}} \Delta X \right) \right]  \notag,
\end{align} 
where $\phi, \alpha$ are the strategies of the forecaster and the adversary respectively. 

Our objective is to study the value function $V^M$ for large $M$ via the limit of the scaled function $u^M$. In order to illustrate the underlying ideas of our main results and define the relevant quantities, we first assume that $u^{M} \to u$ as $M \to \infty$, and that  $u$ is regular enough. According to the Taylor expansion of the right-hand side of \eqref{eq:sdpp}, we obtain that
\begin{align}\label{eq:exdpp}
0&= \min_{\phi}\max_{\alpha} \mathbb{E}^{\phi, \alpha} \left[\sqrt{M}{\nabla u(t,x) \cdot \Delta X}+\partial_t u(t,x) +\frac{1}{2}\sum_{i,j=1}^N \pa_{ij}^2 u(t,x)\Delta X^i\Delta X^j\right]+o(1)\\
&= \max_{\alpha} \min_{\phi}\mathbb{E}^{\phi, \alpha} \left[\sqrt{M}{\nabla u(t,x) \cdot \Delta X}+\partial_t u(t,x) +\frac{1}{2}\sum_{i,j=1}^N \pa_{ij}^2 u(t,x)\Delta X^i\Delta X^j\right]+o(1).
\end{align}
For large enough $M$, in order to have the equality in this expansion, the following conditions have to hold for all $(t,x)\in [0,1)\times \R^N$, 
\begin{align}\label{eq:dpp2}
0&=\min_{\phi}\max_{\alpha}\nabla u(t,x) \cdot \mathbb{E}^{\phi, \alpha} \left[{ \Delta X}\right]=\max_{\alpha} \min_{\phi}\nabla u(t,x) \cdot\mathbb{E}^{\phi, \alpha} \left[{\Delta X}\right].
\end{align}
Additionally, \eqref{eq:translation} yields that 
$$\1\cdot \nabla u=1.$$

Regarding the control of the adversary, assume that there exists $j_0,j_1\in\{1,\ldots,N\}$ so that 
$$(1-a^{j_0}-b^{j_0})\mu^{j_0}+b^{j_0}< (1-a^{j_1}-b^{j_1})\mu^{j_1}+b^{j_1}.$$
We can define the control $\phi$ by 
\begin{align}\label{def:strat}
\phi^j=\pa_j u(t,x)\mbox{ if }j\in\{1,\ldots,N\} \setminus \{j_0,j_1\},
\end{align}
and 
$$\phi^{j_1}=\pa_{x_{j_0}} u(t,x)+\pa_{x_{j_1}} u(t,x),\, \phi^{j_0}=0.$$
Computing 
\begin{align}\label{eq:dpp4}
0&=\max_{\alpha} \min_{\phi}\nabla u(t,x) \cdot\mathbb{E}^{\phi, \alpha} \left[{\Delta X}\right] \notag \\
&=\max_{\alpha} \min_{\phi}\sum_{j=1}^N\left(\pa_j u(t,x) -\phi^j\right)\left((1-a^j-b^j)\mu^j+b^j\right),
\end{align}
this choice of $\phi$ leads to  
\begin{align*}
&\min_{\phi}\sum_{j=1}^N\left(\pa_j u(t,x) -\phi^j\right)\left((1-a^j-b^j)\mu^j+b^j\right)\\
&\leq \pa_{x_{j_0}} u(t,x) \left((1-a^{j_0}-b^{j_0})\mu^{j_0}+b^{j_0}-(1-a^{j_1}-b^{j_1})\mu^{j_1}-b^{j_1}\right). 
\end{align*}
If $ \pa_{x_{j_0}} u(t,x)>0$, we obtain a contradiction with \eqref{eq:dpp4}.  
Therefore, the first order condition \eqref{eq:dpp2} implies that for all $j_0$ so that $ \pa_{x_{j_0}} u(t,x)>0$, we have $\mathbb{E}^{\alpha}[\Delta G^{j_0}]=(1-a^{j_0}-b^{j_0})\mu^{j_0}+b^{j_0}=\sup_j \left((1-a^{j}-b^{j})\mu^{j}+b^{j}\right)=\sup_j\mathbb{E}^{\alpha}[\Delta G^j].$
In order to describe the dynamics of $u$, we define 
\begin{align*}
\cA(p):=\left\{\a\in  \cA: \E^{\alpha}[\Delta G^i]=\sup_j \E^{\alpha}[\Delta G^j] \text{ if $p_i>0$} \right\},
\end{align*}
for any $p \in [0,\infty)^N$.

Therefore, by the heuristic expansion above, one expects that if there is a limit $u$ of $u^M$, then $u$ has to solve
 \begin{align*}
0&=\partial_t u(t,x) +\frac{1}{2}\max_{\alpha \in \cA (\pa u(t,x))}\sum_{i,j=1}^N \pa_{ij}^2 u(t,x) \mathbb{E}^{ \alpha,\pa u (t,x)} \left[\Delta X^i\Delta X^j\right].
\end{align*}
Since $\1 \cdot \nabla u=1$ implies that $\1 \cdot \nabla^2 u=0$, the above equation is equivalent to that 
\begin{align}\label{eq:pde1}
0&=\partial_t u(t,x) +\frac{1}{2}\max_{\alpha \in \cA (\pa u(t,x))}\sum_{i,j=1}^N \pa_{ij}^2 u(t,x) \mathbb{E}^{ \alpha} \left[\Delta G^i\Delta G^j\right].
\end{align}
For notational simplicity, for all $p\in \R^N_+$ and $S\in \bS_N$, we define
\begin{align}\label{eq:H}
H(p,S)&:=\frac{1}{2} \max_{\alpha \in \mathcal{A}(p)} \sum_{i,j=1}^N S_{ij} \E^{\a}[\Delta G^i \Delta G^j],
\end{align}
so that \eqref{eq:pde1} can be written as 
\begin{align}\label{eq:pde}
0&=\partial_t u(t,x)+ H(\nabla u(t,x),\nabla^2 u(t,x)).
\end{align}

Equations of type \eqref{eq:pde} are studied in \cite{MR1100211,MR1119185,MR1924400,MR1988466} in the context of geometric flows. In particular \cite{MR1924400} provides a stochastic representation for geometric flow type equations. Note that our equation \eqref{eq:pde} is not geometric in the sense of \cite[Equation (1.3)]{MR1205984} and our problem can be seen as a deterministic game where the adversary and the forecaster chooses (deterministic controls) in $\cA$ and the simplex of dimension $N$. Thus, in this regard, similar to \cite{doi:10.1002/cpa.20336}, our main results can be seen as representations for the solutions to \eqref{eq:pde} as the limit of deterministic games (whenever wellposedness of \eqref{eq:pde} holds). 

Similar equations also appear in \cite{2020arXiv200813703C,2020arXiv200800052D,2020arXiv200712732D} in the context of prediction. In particular, our Assumption \ref{assume:final} (iii) is inspired by \cite{2020arXiv200800052D} where the authors study the long-time behavior of a prediction problem where the experts are history-dependent and not controlled by the adversary. This point has a fundamental impact on the problem. Indeed, impressively, the limiting equation in \cite{2020arXiv200800052D} is geometric and can be solved by considering the evolution of its level sets. Similarly to \cite{MR4053484,10.5555/2884435.2884474}, in our framework the adversary has to solve a control problem in the long-time regime. Thus, the equation \eqref{eq:pde} is fully nonlinear and in general it is not solvable via geometric methods. However, in some particular cases, we find explicit solutions to \eqref{eq:pde} by finding an optimal control for the adversary; see Section \ref{sec:special}. 

Unlike the various cases in the literature where the generator is continuous on $\R^N-\{0\}$, depending on the specification of $(\mu_i)$, $H$ might fail to be continuous in $p$ on the set $\{(p,S)\in \R_+^N\times \bS_N: p_i=0 \mbox{ for some }i\}$. This lack of continuity has a crucial impact on the wellposedness for viscosity solution of \eqref{eq:pde} and the comparison result for this PDE is not available in the literature. 

Note also that under the Assumption \ref{assume:final} (iii), formally, we have the inequality $ \pa_{x_{j}} u^M(t,x)\geq \frac{\theta}{N}>0$ for all $j\in \{1,\ldots, N\}$. Thus, in this case, one expects that 
\begin{align*}
\cA(\nabla u(t,x)):=\left\{\a\in  \cA: \E^{\alpha}[\Delta G^i]=\sup_j \E^{\alpha}[\Delta G^j],\, \forall i\right\},
\end{align*}
and the set of strategies for the adversary yields the balanced strategies defined in \cite{10.5555/2884435.2884474}.
\begin{definition}
$\mathcal{A}_B$ denotes the set of ``balanced" strategies $\alpha$ for the adversary, i.e., strategies $\a\in \cA$ satisfying 
\begin{align}
 \mathbb{E}^{\alpha} [\Delta G^{j_0}]= \mathbb{E}^{ \alpha} [\Delta G^{j_1}]
\end{align}
for all $j_0,j_1\in \{1,\ldots,N\}$. For any $\alpha \in \mathcal{A}_B$, we define 
\begin{align}\label{def:c}
c_{\alpha}:=\E^{\alpha}[\Delta G^i]=(1-a^i-b^i)\mu^i+b^i, \text{ for any $i=1,\dotso, N$.}
\end{align}

\end{definition}
Note that for any $p\in (0,\infty)^N$, $\cA(p)=\cA_B$ no matter this set is empty or not. 
We provide a necessary and sufficient condition on $(\mu_i)$ for the existence of balanced strategies. 
\begin{proposition}\label{prop:existbalanced}
The set of balanced strategies $\mathcal{A}_B$ is not empty if and only if 
\begin{align}\label{eq:existbalanced}
\inf_{c \in [0,1]} \sum_{i=1}^N \left(\frac{\mu^i-c}{\mu^i} \vee \frac{c-\mu^i}{1-\mu^i} \right) \leq 1,
\end{align}
 where we make the convention $\frac{0}{0}=0$. 
\end{proposition}
\begin{proof}
The proof of is provided in the Appendix.
\end{proof}
For notational simplicity, we also define the generator, 
\begin{align}\label{eq:HB}
H_B(S)&:=\frac{1}{2} \max_{\alpha \in \mathcal{A}_B} \sum_{i,j=1}^N S_{ij} \E^{\a}[\Delta G^i \Delta G^j],
\end{align}
whose definition is motivated by the fact that  
\begin{align*}
H^*(p,S)&:=\limsup_{(q,R)\to (p,S)}H(q,R)=H(p,S)\mbox{ and}\\
H_*(p,S)&:=\liminf_{(q,R)\to (p,S)}H(q,R)=H_B(S)\mbox{ for all }p\in \R^N_+ \text{ if $\cA_B\neq \emptyset$}.
\end{align*}

Given this discontinuity of the generator, we provide here the definition of viscosity solutions which is also available in \cite{MR1119185}.
\begin{definition}\label{def:viscosity}
An upper (resp. lower) semicontinuous function $u$ is a viscosity subsolution (resp. supersolution) of \eqref{eq:pde} if for all $(t,x)\in [0,1)\times \R^N$ and smooth function $\phi$ so that $u-\phi$ has a local maximum (resp. minimum) at $(t,x)$, we have that 
$$-\pa_t \phi(t,x)-H^*(\nabla\phi(t,x),\nabla^2\phi(t,x) )\leq 0$$
$$(resp. -\pa_t \phi(t,x)-H_*(\nabla\phi(t,x),\nabla^2\phi(t,x) )\geq 0).$$
 
\end{definition}

\section{Main results}\label{sec:main}
In this section, we provide the main results regarding the growth of regret and asymptotically optimal strategies of the forecaster and the adversary. The results fundemantally depend on whether $\cA_B=\emptyset $ or not.

\subsection{Growth of regret for the case $\mathcal{A}_B \not= \emptyset$}
We assume in this subsection that $\mathcal{A}_B \not= \emptyset$. We prove the following priori bound for $u^M$. 
\begin{lemma}\label{lem:lineargrowth}
Assume that Assumption \ref{assume:final} (i) holds. Then, there exists a constant C independent of $M$ such that for all $(t,x)\in [0,1]\times \R^N$
\begin{align*}
|u^M(t, x)- \Phi(x) | \leq C(2-t).
\end{align*}
\end{lemma}
\begin{proof}
We provide the proof of the Lemma in the Appendix. 
\end{proof}
Given Lemma \ref{lem:lineargrowth}, we define the functions 
\begin{align}\label{eq:semilimits}
\overline{u}(t,x):=\limsup\limits_{(M,s,y) \to (\infty,t,x)} u^M(s,y)\\
\underline{u}(t,x):=\liminf\limits_{(M,s,y) \to (\infty,t,x)} u^M(s,y).
\end{align}
The following comparison principle is a special case of \cite[Theorem 2.1]{MR1119185} backwards in time.
\begin{lemma}\label{comp}Under Assumption \ref{assume:final} (i)-(ii) and 
subject to the final condition $U(1,x)=\Phi(x)$, there exists a unique viscosity solution to 
\begin{align}\label{eq:pde2}
0&=\partial_t U(t,x)+ H_B(\nabla^2 U(t,x))
\end{align}
 that grows at most linearly and is uniformly continuous. We will denote this unique solution by $U$. Moreover, if $u_1$ is a subsolution, and $u_2$ is a supersolution, then comparison principle holds, i.e., $u_1 \leq U\leq u_2$ on $[0,1]\times \R^N$. 
\end{lemma}
\begin{proof}
The result is a direct consequence of \cite[Theorem 2.1]{MR1119185}.
\end{proof}
Thanks to the identity $H_*(p,S)=H_B(S)$, any supersolution to \eqref{eq:pde} is also a supersolution to \eqref{eq:pde2}. In the following theorem, using this property, we show that $U$ provides a lower bound for the scaled value function. 
\begin{theorem}\label{thm:viscosity2}
Assume that Assumption \ref{assume:final} (i) and (ii) holds. Then, $\underline{u}$(resp. $\overline u$) is a supersolution (resp subsolution) of \eqref{eq:pde} subject to the terminal condition $\underline{u}(1,x)=\overline{u}(1,x)= \Phi (x)$, and hence the solution of \eqref{eq:pde2} provides a lower bound to the growth of regret as 
\begin{align}\label{eq:lowerb}
\liminf_{M\to \infty} \frac{1}{\sqrt{M}} V^M \left(\lceil Mt \rceil, \sqrt{M}x \right)\geq\underline{u}(t,x) \geq U(t,x),
\end{align}
where $U$ is the unique viscosity solution to \eqref{eq:pde2}.
\end{theorem}

\begin{proof} 
The proof is almost the same as \cite[Theorem 2.1]{MR1115933} and \cite[Theorem 7]{MR4053484}, and we only indicate our modifications. 
We first show the supersolution property of $\underline u$. Let $(t_0,x_0)\in [0,1)\times \R^N$ and $\psi$ smooth so that 
$\underline u-\psi$ has a strict local minimum at $(t_0,x_0)$. Then, similarly to \cite[Theorem 2.1]{MR1115933}, there exists $M_n\to 0$ and 
$(s_n,y_n)\to (t_0,x_0)$ satisfying
$$u^{M_n}(s_n,y_n)\to \underline u(t_0,x_0)\mbox{ and }u^{M_n}-\psi\mbox{ has a local minimum at }(s_n,y_n).$$
Denote $\xi_n=u^{M_n}(s_n,y_n)-\psi(s_n,y_n)$ that converges to $0$. The dynamic programming principle \eqref{eq:sdpp} and the minimality condition for $u^{M_n}-\psi$ yields that 
\begin{align*}
\xi_n&=u^{M_n}(s_n,y_n)-\psi(s_n,y_n)\\
&=\min_{\phi}\max_{\alpha} \mathbb{E}^{\phi,\alpha} \left[u^M\left( s_n+\frac{1}{M_n}, y_n+\frac{1}{\sqrt{M_n}} \Delta X \right) -\psi(s_n,y_n)\right]\\
&\geq \min_{\phi}\max_{\alpha} \mathbb{E}^{\phi,\alpha} \left[\psi\left( s_n+\frac{1}{M_n}, y_n+\frac{1}{\sqrt{M_n}} \Delta X \right) -\psi(s_n,y_n)+\xi_n\right]
\end{align*}
where we used the minimality of $u^{M_n}-\psi$ to obtain the inequality. Given that $\psi$ is fixed, we can now proceed to expand as in \eqref{eq:exdpp} to have 
$$o(1)\geq \min_{\phi}\max_{\alpha\in \cA} \mathbb{E}^{\phi, \alpha} \left[\sqrt{M_n}{\nabla \psi(s_n,y_n)\cdot \Delta X}+\partial_t \psi(s_n,y_n) +\frac{1}{2}\sum_{i,j=1}^N \pa_{ij}^2 \psi(s_n,y_n)\Delta X^i\Delta X^j\right].
$$
By restricting the choice of $\alpha$ this inequality in particular implies that 
\begin{align*}
o(1)&\geq \min_{\phi}\max_{\alpha\in \cA_B} \mathbb{E}^{\phi, \alpha} \left[\sqrt{M_n}{\nabla \psi(s_n,y_n)\cdot \Delta X}+\partial_t \psi(s_n,y_n) +\frac{1}{2}\sum_{i,j=1}^N \pa_{ij}^2 \psi(s_n,y_n)\Delta X^i\Delta X^j\right]\\
&\geq \max_{\alpha\in\cA_B} \mathbb{E}^{\phi, \alpha} \left[\partial_t \psi(s_n,y_n) +\frac{1}{2}\sum_{i,j=1}^N \pa_{ij}^2 \psi(s_n,y_n)\Delta X^i\Delta X^j\right]
\end{align*}
where we used the fact that for a balanced strategy the regret does not depend on the strategy of the forecaster. 
$u^M$ satisfies 
\begin{align}\label{inv2}
u^M(s,y+\lambda \1)=u^M(s,y)+\lambda
\end{align} 
for all $(s,y)\in [0,1]\times \R^N$ and $\lambda \in \R$, and $u^{M_n}-\psi$ has a local minimum at $(s_n,y_n)$. Thus, $\psi$ satisfies $\nabla \psi(s_n,y_n)\cdot \1=1$. Therefore, similarly as in \eqref{eq:pde1}, we easily obtain that 
\begin{align*}
o(1)&\geq \partial_t \psi(s_n,y_n) + H_B(\nabla^2 \psi(s_n,y_n) )=\pa_t  \psi(s_n,y_n)+H_*(\nabla \psi(s_n,y_n),\nabla^2 \psi(s_n,y_n) ).
\end{align*}
The convergence of $(s_n,y_n)$, and the continuity of $H_B$ concludes the proof of the super solution. 

We now prove the subsolution property of $\overline u$. Similarly as above, for a given $(t_0,x_0)\in [0,1)\times \R^N$ and $\psi$ smooth so that 
$\underline u-\psi$ has a strict local maximum at $(t_0,x_0)$, we can establish that 
\begin{align}\label{eq:dppsub}
o(1)\leq \max_{\alpha\in \cA} \min_{\phi}\left\{\sqrt{M_n}\sum_{j=1}^N\left(\pa_j \psi(s_n,y_n) -\phi^j\right)\left((1-a^j-b^j)\mu^j+b^j\right)\right.\\
\left.+\partial_t \psi(s_n,y_n)+\frac{1}{2}\sum_{i,j=1}^N \pa_{ij}^2 \psi(s_n,y_n)\mathbb{E}^{\phi, \alpha} \left[\Delta X^i\Delta X^j\right]\right\}.\notag
\end{align}
Let $\a\notin \cA(\nabla \psi(s_n,y_n))$, then there exists $i\in \{1,\ldots,N\}$ so that $\pa_i \psi(s_n,y_n)>0$ and 
$$(1-a^i-b^i)\mu^i+b^i<\sup_j (1-a^j-b^j)\mu^j+b^j.$$ 
Similarly as \eqref{def:strat}, for such a strategy one can find strategy $\phi$ for the forecaster so that 
$$\sum_{j=1}^N\left(\pa_j \psi(s_n,y_n) -\phi^j\right)\left((1-a^j-b^j)\mu^j+b^j\right)\leq -\epsilon<0$$
for all $n$ large enough. Thus, the maximum in \eqref{eq:dppsub} cannot be achieved at such a strategy for $n$ large enough. Therefore, \eqref{eq:dppsub} yields
\begin{align*}
o(1)&\leq \max_{\alpha\in \cA(\nabla \psi(s_n,y_n))} \min_{\phi}\left\{\sqrt{M_n}\sum_{j=1}^N\left(\pa_j \psi(s_n,y_n) -\phi^j\right)\left((1-a^j-b^j)\mu^j+b^j\right)\right.\\
&\left.+\partial_t \psi(s_n,y_n)+\frac{1}{2}\sum_{i,j=1}^N \pa_{ij}^2 \psi(s_n,y_n)\mathbb{E}^{\phi, \alpha} \left[\Delta X^i\Delta X^j\right]\right\}\\
&\leq \max_{\alpha\in \cA(\nabla \psi(s_n,y_n))} \left\{\sqrt{M_n}\sup_j \left((1-a^j-b^j)\mu^j+b^j\right) \left( \sum_{j=1,\,  \pa_{j} \psi(s_n,y_n)>0}^N\pa_j \psi(s_n,y_n) -1\right)\right.\\
&\left.+\partial_t \psi(s_n,y_n)+\frac{1}{2}\sum_{i,j=1}^N \pa_{ij}^2 \psi(s_n,y_n)\mathbb{E}^{\alpha} \left[\Delta G^i\Delta G^j\right]\right\}\\
&\leq\partial_t \psi(s_n,y_n)+ \max_{\alpha\in \cA(\nabla \psi(s_n,y_n))} \frac{1}{2}\sum_{i,j=1}^N \pa_{ij}^2 \psi(s_n,y_n)\mathbb{E}^{ \alpha} \left[\Delta G^i\Delta G^j\right]
\end{align*}
where we use the fact that 
$$\sum_{j=1,\,  \pa_{j} \psi(s_n,y_n)>0}^N \pa_j \psi(s_n,y_n) =1,$$
 and 
$$\sum_{i,j=1}^N \pa_{ij}^2 \psi(s_n,y_n)\mathbb{E}^{\phi, \alpha} \left[\Delta X^i\Delta X^j\right]=\sum_{i,j=1}^N \pa_{ij}^2 \psi(s_n,y_n)\mathbb{E}^{ \alpha} \left[\Delta G^i\Delta G^j\right]$$
due to $\nabla \psi(t_n,y_n)\cdot \1=1$.
Thus, we finally obtain that 
$$o(1) \leq \partial_t \psi(s_n,y_n)+H(\nabla \psi(s_n,y_n),\nabla^2 \psi(s_n,y_n)),$$
which leads to the subsolution property
$$0 \leq \partial_t \psi(t_0,x_0)+H^*(\nabla \psi(t_0,x_0),\nabla^2 \psi(t_0,x_0)).$$

Given the supersolution property of $\underline u$, the identity $H_*(p,S)=H_B(S)$ and the comparison result in Lemma~\ref{comp}, we easily have  that $\underline u\geq U$ which implies \eqref{eq:lowerb}.
\end{proof}
\begin{remark}
Although, it is mathematically appealing to have a comparison result for the PDE \eqref{eq:pde}, we do not need it for practical problems such as lower bound of growth of regret such as \eqref{eq:lowerb}. The lower bound of regret is a consequence of supersolution property of $\underline u$ and a comparison result for the PDE \eqref{eq:pde2} (which is  a significantly simpler task than a comparison for \eqref{eq:pde}). 

\end{remark}

\begin{remark}\label{rmk:upperbound}
The classical online problem easily yields an upper bound. Indeed, in this problem the adversary decides on the distribution of experts' predictions at each round, i.e., the adversary chooses an element in the probability space over $\{0,1\}^N$, see  \cite{MR4120922,2019arXiv190202368B,MR4053484}, etc. Denote the value function of this game by $W^M(m,x)$. According to \cite[Theorem 7]{MR4053484}, we have that 
\begin{align*}
\lim\limits_{M \to \infty} \frac{1}{\sqrt{M}} W^M\left( \lceil Mt \rceil, \sqrt{M}x \right) = w(t,x),
\end{align*}
where $w(t,x)$ is the viscosity solution to 
\begin{align}\label{pde:forsup}
&\pa_t w(t,x)+ \frac{1}{2} \max_{v \in \{0,1\}^N} \langle \pa^2_{xx} w (t,x) \cdot v, v \rangle =0, \\
&w(1,x)=\Phi(x).
\end{align}
Since the adversary in this game fully controls the the prediction of experts, it can be easily seen that $V^M(m,x) \leq W^M(m,x)$, and therefore we obtain that 
\begin{align}\label{exp:forsup}
V^M(\lceil Mt \rceil, \sqrt{M}x) \leq w(t,x) \sqrt{M}+ o(\sqrt{M}). 
\end{align}
\end{remark}

Both in the classical problem in Remark \ref{rmk:upperbound} and in the description of the lower bound function $u$, the set of strategies of the adversary are balanced. We now provide a counter example that shows that without Assumption \ref{assume:final} (iii) it might be optimal for the adversary to choose a non balanced strategy by exhibiting a case where \eqref{eq:lowerb} is strict. Thus, unlike in \cite{10.5555/2884435.2884474}, with corruption, the optimal strategy of the adversary is not always balanced. This example also shows that, in general, $\overline u$ can not be a subsolution to \eqref{eq:pde2}, but has to be characterized as a subsolution to \eqref{eq:pde}. We will show in Theorem \ref{thm:viscosity1} that Assumption \ref{assume:final} (iii) is in fact sufficient to obtain that $\underline u=\overline u$ and solves  \eqref{eq:pde2}.
\begin{example}\label{eq:counter}
For $N=3$, $\mu^1=0, \mu^2=\mu^3=1$, it can be easily verified that $\mathcal{A}_B=\{(a^i,b^i): b^1=1\}$, i.e., the adversary always corrupts the first expert, and set his gain to $1$. Then the viscosity solution of \eqref{eq:pde2} is $U(t,x)=\Phi_m(x)=\max_i x^i$. However, if the adversary chooses the strategy $(a^2=a^3=1/2)$, then we have that $u^M(t,x)>\Phi_m(x)$ for any $t \in [0,1)$. Therefore $\limsup_{M \to \infty} u^M(t,x)$ cannot always be a subsolution of \eqref{eq:pde2}.
\end{example}

The following Theorem and Example \ref{eq:counter} show the importance of Assumption \ref{assume:final} (iii), which allows us to obtain the exact growth rate of regret. With Assumption \ref{assume:final} (iii), formally we obtain that $\nabla u^M \in (\theta/N, +\infty)^N$, and thus the adversary is forced to use balanced strategies. Therefore, we can show that the scaled value function converges to the solution $U$ of \eqref{eq:pde2}. 
\begin{theorem}\label{thm:viscosity1}
Assume that Assumption \ref{assume:final} (i), (ii)  and (iii) hold. Then, $\underline{u}$(resp. $\overline u$) is a  lower (resp. upper) semicontinuous viscosity supersolution (resp. subsolution) of \eqref{eq:pde2} subject to the terminal condition $\underline{u}(1,x)=\overline{u}(1,x)= \Phi (x)$. Therefore, $\overline u=\underline u=U$ provides the growth rate of regret as $$ V^M \left(\lceil Mt \rceil, \sqrt{M}x \right)= U(t,x)\sqrt{M}+o(\sqrt{M}).$$
\end{theorem}
\begin{proof}
By \eqref{eq:strict}, we have that
\begin{align*}
u^M(t,x+y) &= \min_{\phi}\max_{\alpha} \mathbb{E}^{\phi,\alpha} \left[u^M\left( t+\frac{1}{M}, x+y+\frac{1}{\sqrt{M}} \Delta X \right) \right] \\
& \geq \min_{\phi}\max_{\alpha} \mathbb{E}^{\phi,\alpha} \left[u^M\left( t+\frac{1}{M}, x+\frac{1}{\sqrt{M}} \Delta X \right) \right] + \frac{\theta}{N} \langle y , \1 \rangle  \\
& = u^M(t,x) + \frac{\theta}{N} \langle y , \1 \rangle. 
\end{align*}
Therefore if $\underline{u}-\psi$ or $\bar{u}-\psi$ attains a local extreme at $(t_0,x_0) \in [0,1) \times \mathbb{R}^N$,  it follows that $\nabla \psi(t_0,x_0) \in \left[\frac{\theta}{N},+\infty\right)^N$. Then, following the same arguments as in the proof of Theorem~\ref{thm:viscosity2},
we obtain the sign of $$ \partial_t \psi(s_n,y_n)+H(\nabla \psi(s_n,y_n),\nabla^2 \psi(s_n,y_n))$$
for $(s_n,y_n)\to (t_0,x_0)$. The conclusion $\nabla \psi(t_0,x_0) \in \left[\frac{\theta}{N},+\infty\right)^N$ and the identity $H(p,S)=H_B(S)$ if $p_i>0$ for all $i$ allows us to obtain the sign of  
$$ \partial_t \psi(s_n,y_n)+H(\nabla \psi(s_n,y_n),\nabla^2 \psi(s_n,y_n))= \partial_t \psi(s_n,y_n)+H_B(\nabla^2 \psi(s_n,y_n)).$$
Therefore we obtain the required viscosity property. The conclusion of the theorem follows by the Lemma \ref{comp}.
\end{proof}

Given the solution $U$ of \eqref{eq:pde2}, we design strategies for the adversary. For a fixed maturity $M$, denote $\tilde{x}:=\frac{x}{\sqrt{M}}, t_m:= \frac{m}{M}$. We define the strategy for the adversary
\begin{align}\label{eq:asymptoticadversary}
 \boldsymbol{\alpha}^*=(\a^*_1, \dotso, \a^*_{M})
\end{align}
  via $$\a^*_m(x)= \argmax_{\a \in \mathcal{A}_B} \sum_{i,j=1}^N  \pa_{ij}^2 U(t_{m-1}, \tilde{x}) \E^{\a}[\Delta G^i \Delta G^j] .$$
Define 
\begin{align*}
\underline{V}^M(0,x)=\inf_{\boldsymbol{\phi}}\E^{\boldsymbol{\phi}, \boldsymbol{\a}^*}[\Phi(X_M) \, | \, X_0=x],
\end{align*}
where $\bs{\phi}=(\phi_1,\dotso,\phi_M)$ is any strategy of the forecaster. In the next proposition, we will show that 
\begin{align}\label{eq:asymptotic}
\lim\limits_{M \to \infty} \frac{1}{\sqrt{M}}  \underline{V}(0,\sqrt{M}x)  \geq U(0,x),
\end{align}
under assumptions on $U$. In Section~\ref{sec:special}, we will verify these assumptions for a special case. 
\begin{proposition}\label{prop:asymptotic}
Assume that Assumption \ref{assume:final} (i) and (ii) hold. Suppose the solution $U$ to \eqref{eq:pde2} is smooth and  satisfies the derivative bounds 
\begin{align}\label{eq:boundder}
|\pa_{tt}^2 U (1-t,x)| \leq \frac{C}{{t}^{\frac{3}{2}}}, \quad |\pa_{tx}^2 U (1-t,x)| \leq \frac{C}{t}, \quad  |\pa_{xxx}^3 U (1-t,x)| \leq \frac{C}{t}, \quad \text{ $\forall x \in \mathbb{R}^N$,}
\end{align}
for some positive constant $C$. Then \eqref{eq:asymptotic} holds. Therefore, according to Theorem~\ref{thm:viscosity2} the asymptotic strategy $\bs{\a}^*$ in \eqref{eq:asymptoticadversary} for the adversary guarantees $U$ as a lower bound of regret, i.e.,
\begin{align*}
\lim\limits_{M \to \infty} \frac{1}{M}V^M(0,\sqrt{M}x) \geq \lim\limits_{M \to \infty} \frac{1}{M} \underline{V}^M(0,\sqrt{M}x) \geq U(0,x). 
\end{align*} 

\end{proposition}
\begin{proof}
It can be easily verified that 
\begin{align*}
\frac{1}{\sqrt{M}} \underline{V}^M(0, \sqrt{M}x) - U(0,x)=& \frac{\inf_{\bs{\phi}}\E^{\bs{\phi}, \bs{\a}^*}[\Phi(X_M) \, | \, X_0=\sqrt{M}x  ]}{\sqrt{M}}-U(0,x)\\
=& \inf_{\bs{\phi}}\E^{\bs{\phi}, \bs{\a}^*}[ U(1,\tilde{X}_M)\, | \, \tilde{X}_0=x  ]-U(0,x) \\
=& \inf_{\bs{\phi}} \left( \sum_{m=1}^{M} \E^{\phi,\a_{m}^*}[ U(t_{m}, \tilde{X}_{m})- U(t_{m-1} ,\tilde{X}_{m-1}) \, | \, \tilde{X}_0=x ] \right).
\end{align*}
Note that 
\begin{align}
&\E^{\phi,\a}  \left[U\left(t_m,\tilde{X}_m\right)-U\left(t_{m-1},\tilde{X}_{m-1}\right)\,|\,\tilde{X}_{m-1}=\tilde{x}_{m-1} \right] \label{eq:discrete-1}   \\
&=\E^{\phi,\a} \left[\pa_x U\left(t_{m-1},\tilde{x}_{m-1}\right)^\top\Delta \tilde{X}_{m} \right]  \label{eq:discrete0} \\
&+2\E^{\phi,\a} \left[\int_0^{\sqrt{\frac{1}{M}}} \left(\sqrt{\frac{1}{M}}-s\right)\left(\pa_t U+ \frac{1}{2} \Delta X_m^\top \cdot \pa^2_{xx}U \cdot \Delta X_m\right)(t_{m-1},\tilde{x}_{m-1}+s\Delta X_m) ds \right]  \label{eq:discrete1} \\
& +2\E^{\phi,\a}  \left[\int_0^{\sqrt{\frac{1}{M}}}  \left(\sqrt{\frac{1}{M}}-s\right)\left(\pa_t U(t_{m-1},\tilde{X}_m)-\pa_t U(t_{m-1},\tilde{x}_{m-1}+s\Delta X_m)  \right)  ds \right] \label{eq:discrete2} \\
&+\E^{\phi,\a} \left[\int_0^{\frac{1}{{M}}} \left( \pa_t U(t_{m-1}+s, \tilde{X}_m)-\pa_t U(t_{m-1},\tilde{X}_m)\right)ds \right] \label{eq:discrete3}.
\end{align}

Under the strategy $\a_{m}^*$,  the term \eqref{eq:discrete0} is zero. Since $\pa_{xx}^2 U \cdot \1=0$, the term \eqref{eq:discrete1} is independent of $\phi$.
Due to our choice of $\a_{m}^*$, we have that 
\begin{align*}
\E^{\phi,\a_{m}^*}\left[\left(\pa_t U+ \frac{1}{2} \Delta X_m^\top \cdot \pa^2_{xx}U \cdot \Delta X_m\right)(t_{m-1},\tilde{x}_{m-1})\right]=0.
\end{align*}
According to the derivatives bounds \eqref{eq:boundder}, it can be easily seen that  
\begin{align*}
 & \left(\pa_t U+ \frac{1}{2} \Delta X_m^\top \cdot \pa^2_{xx}U \cdot \Delta X_m\right)(t_{m-1},\tilde{x}_{m-1}+s\Delta X_m)  \\
 & \geq \left(\pa_t U+ \frac{1}{2} \Delta X_m^\top \cdot \pa^2_{xx}U \cdot \Delta X_m\right)(t_{m-1},\tilde{x}_{m-1})-\frac{Cs}{1-t_{m-1}}.
\end{align*}
Therefore the term \eqref{eq:discrete1} is bounded below by 
\begin{align*}
-2C \int_0^{\sqrt{\frac{1}{M}}} \frac{\left(\sqrt{\frac{1}{M}}-s\right)s}{1-t_{m-1}} ds= -\frac{C}{(1-t_{m-1}){M}^{\frac{3}{2}}},
\end{align*}
where $C$ is allowed to change from line to line. Similarly, it can be easily verified that the term \eqref{eq:discrete2}  is bounded below by $-\frac{C}{(1-t_{m-1}){M}^{\frac{3}{2}}}$. As a result of \eqref{eq:boundder}, we have that 
\begin{align*}
\pa_t U(t_{m-1}+s, \tilde{X}_m)-\pa_t U(t_{m-1},\tilde{X}_m) & \geq -C \int_{0}^s \frac{1}{(1-t_{m-1}-w)^{\frac{3}{2}}} dw ,
\end{align*}
and therefore the  term \eqref{eq:discrete3} is bounded below by 
\begin{align*}
-C \int^{\frac{1}{M}}_0\int_{0}^s \frac{1}{(1-t_{m-1}-w)^{\frac{3}{2}}} dw ds &=-C \int^{\frac{1}{M}}_0\int_{w}^{\frac{1}{M}} \frac{1}{(1-t_{m-1}-w)^{\frac{3}{2}}} ds dw \\
&=-C \int_{0}^{\frac{1}{M}} \frac{\frac{1}{M}-s}{(1-t_{m-1}-s)^{\frac{3}{2}}} ds .
\end{align*}

Putting together all the estimates for \eqref{eq:discrete0}, \eqref{eq:discrete1}, \eqref{eq:discrete2} and \eqref{eq:discrete3} above, we conclude that 
\begin{align*}
& \E^{\phi,\a_{m}^*}  \left[U\left(t_m,\tilde{X}_m\right)-U\left(t_{m-1},\tilde{X}_{m-1}\right)\,|\,\tilde{X}_{m-1}=\tilde{x}_{m-1} \right] \\
&\geq -C\left( \frac{1}{(1-t_{m-1}){M}^{\frac{3}{2}}}+\int_{0}^{\frac{1}{M}} \frac{\frac{1}{M}-s}{(1-t_{m-1}-s)^{\frac{3}{2}}} ds\right).
\end{align*}
It can be easily verified that 
\begin{align*}
\lim\limits_{M \to \infty} \sum_{m=1}^M \left( \frac{1}{(1-t_{m-1}){M}^{\frac{3}{2}}}+\int_{0}^{\frac{1}{M}} \frac{\frac{1}{M}-s}{(1-t_{m-1}-s)^{\frac{3}{2}}} ds\right)=0,
\end{align*}
and therefore 
\begin{align*}
\lim\limits_{M \to \infty} \frac{1}{\sqrt{M}} \underline{V}^M(0, \sqrt{M}x) - U(0,x) \geq 0.
\end{align*}
\end{proof}

One might expect that the function $U$ captures important features of the problem, and the algorithm of the forecaster given  by 
$$\phi^*_m=\{\pa_j U (t_{m-1}, \tilde{X}_{m-1} )\}_{j=1}^N$$ yields the best algorithm for the growth of the regret, i.e., an equality holds in \eqref{eq:asymptotic}. 
Such a conjecture holds in \cite{MR4120922,2020arXiv200813703C,kobzar2020new,pmlr-v125-kobzar20a}. Unfortunately, as proved by the following counter example, in our case $\phi^*_m$ does not provide an asymptotic optimal algorithm. 

\vspace{4pt}

{\bf{Counter Example:}} Consider the case $N=2$, $\mu^1=\frac{3}{4}, \mu^2=\frac{1}{4}$ with final condition $\Phi=\Phi_{m,\theta}$. Let $U$ be the viscosity solution to \eqref{eq:pde2}. Then it holds that 
\begin{align*}
\lim\limits_{M \to \infty} \frac{1}{\sqrt{M}} \sup_{\boldsymbol{\a}}\E^{\boldsymbol{\phi}^*, \boldsymbol{\a}}[\Phi(X_M) \, | \, X_0=\sqrt{M}x]>U(0,x),
\end{align*}
i.e., $\phi^*_m$ is not asymptotically optimal. According to Proposition~\ref{prop:special} in the next section, it can be easily verified that \eqref{eq:pde2} becomes 
\begin{align*}
\pa_t U (t,x)+\frac{3}{32}(\pa^2_{11} U(t,x)+\pa^2_{22} U(t,x))=0.
\end{align*}
From $\pa^2_{11}U + \pa^2_{12}U=\pa^2_{12} U+ \pa^2_{22} U =0$, we deduce that $$\pa^2_{11}U=\pa^2_{22} U = - \pa^2_{12} U=-\pa^2_{21} U,$$
and hence 
\begin{align*}
\pa_t U (t,x)+\frac{3}{16}\pa^2_{11} U(t,x)=0.
\end{align*}
By Feynman-Kac representation of $U$, it can be easily verified $\pa_t U \leq 0, \pa^2_{11}U \geq 0$.  By choosing $b^1=a^2=\frac{1}{2}$, we obtain that 
\begin{align*}
&\pa_t U(t,x)+ \frac{1}{2} \max_{\a} \sum_{i,j} \pa_{ij}^2 U(t,x) \E^{\a}[ \Delta G^i \Delta G^j]\\
&=\pa_t U(t,x)+ \frac{3}{8} \pa^2_{11} U(t,x)= -\pa_t U(t,x).
\end{align*}
Therefore we obtain that 
\begin{align*}
\sup_{\a} \E^{\phi_{m}^*,\a}  \left[U \left(t_m,\tilde{X}_m\right)-U \left(t_{m-1},\tilde{X}_{m-1}\right)\,|\,\tilde{X}_{m-1}=\tilde{x}_{m-1} \right]= - \frac{\pa_t U (t_{m-1},\tilde{x}_{m-1})}{M}+o(1/M).
\end{align*}
Due to the explicit formula of $U$, we obtain that 
\begin{align*}
-\pa_t U(1-t,x^1,x^2) \geq c {t}^{-1/2} e^{-\frac{(x^1-x^2)^2}{4dt}}
\end{align*}
for some positive constants $c,d$. To close the argument, we need an estimate of $\tilde{x}^1_{m-1}-\tilde{x}^2_{m-1}$.

Define the strategy $\hat{\bs{\a}}=(\hat{\a}_1, \dotso, \hat{\a}_{M})$ such that 
\begin{align*}
\hat{\a}_m=\{b^1_m=a^2_m=\frac{1}{2}\}. 
\end{align*}
Under the strategy $\hat{\bs{\a}}$, $Z_m:=X^1_{m}-X^2_{m}$ becomes a random walk with
\begin{align*}
\E^{\hat{\bs{\a}}}[Z_m]=\frac{3}{4}, \quad \text{Var}^{\hat{\bs{\a}}}[Z_m]=\frac{3}{16}.
\end{align*}
Therefore, the scaled random walk $(t_{m},\tilde{Z}_m)$ converges to a drifted Brownian motion $(B_t)_{t \geq 0}$ such that 
\begin{align*}
\E[B_t]=\frac{3t}{4}, \quad \text{Var}[B_t]=\frac{3t}{16}. 
\end{align*}
Since $\frac{B_t - \frac{3t}{4}}{\sqrt{\frac{3t}{16}}}$ has standard normal distribution, we define 
\begin{align*}
p:= \P\left[\frac{|B_t - \frac{3t}{4}|}{\sqrt{\frac{3t}{16}}} \leq 1\right] = \P\left[\frac{3t}{4}-\sqrt{\frac{3t}{16}} \leq B_t \leq \frac{3t}{4}+\sqrt{\frac{3t}{16}}\right].
\end{align*}
Therefore, we obtain the estimate 
\begin{align*}
\lim\limits_{M \to \infty}&  \frac{1}{\sqrt{M}}\E^{\boldsymbol{\phi}^*, \hat{\boldsymbol{\a}}}[\Phi(X_M) \, | \, X_0=\sqrt{M}x]-U(0,x) \\&=\lim\limits_{M \to \infty} \sum_{m=1}^{M} \E^{\bs{\phi}^*,\bs{\hat{\a}}}[ U(t_{m}, \tilde{X}_{m})- U(t_{m-1} ,\tilde{X}_{m-1}) \, | \, \tilde{X}_0=0 ]\\
& = \E\left[\int_0^1 c (1-t)^{-1/2} e^{-\frac{B_t^2}{4d(1-t)}}    dt                    \right]  \\
& \geq p \int_0^1 c (1-t)^{-1/2} e^{-\frac{(\frac{3t}{4}+\sqrt{\frac{3t}{16}})^2}{4d(1-t)}} dt >0.
\end{align*}

\vspace{10pt}

 {\bf{Algorithm for the forecaster:}} The decomposition in \eqref{eq:discrete-1} and the identity \eqref{eq:dpp4} show that the algorithm for the forecaster defined as $\phi^*_m=\{\pa_j U (t_{m-1}, \tilde{X}_{m-1} )\}_{j=1}^N$ would set to $0$ the term \eqref{eq:discrete0}. However, even if $u$ solves \eqref{eq:pde} but not \eqref{pde:forsup}, we cannot control the sign of \eqref{eq:discrete1} and there exist strategies for the adversary that renders $u(t_{m-1}, \tilde{X}_{m-1} )$ a submartingale (instead of a supermartingale). Thus, $\bs{\phi}^*=(\phi^*_1,\dotso,\phi^*_M)$ is not the best strategy for the learner and $\nabla u$ does not necessarily provides the best algorithm for the forecaster. 

However, if we assume that $\psi$ is a smooth supersolution to \eqref{pde:forsup}, it can be easily verified that $\phi^*_m=\{\pa_j \psi(t_{m-1}, \tilde{X}_{m-1} )\}_{j=1}^N$
provides an algorithm for the forecaster for which the growth of the regret can be bounded from above as in \eqref{exp:forsup}. 

\subsection{Growth of regret when $\mathcal{A}_B = \emptyset$} 
We now assume that $\mathcal{A}_B = \emptyset$.
In this case, we cannot rely on the PDE \eqref{eq:pde2} to obtain the growth of the regret and we have to introduce some auxiliary functions. Without the Assumption \ref{assume:final} (iii), we provide an example showing that the regret is also of order $\sqrt{M}$. The following result holds no matter $\mathcal{A}_B$ is empty or not. 

\begin{proposition}\label{prop:emptycase1}
Assume that $\Phi(x)=\Phi_m(x)=\max_i x^i$. Then, there exists a function $\hat u:[0,1]\times \R^2\mapsto \R$ solving a linear parabolic non-degenerate PDE with constant coefficients and terminal condition $\hat{u}(1,x)=x^1\vee x^2$ for $x\in \R^2$ so that
\begin{align*}
\liminf\limits_{M \to \infty} \frac{V^M(0,\sqrt{M}x)}{\sqrt{M}} \geq \hat{u}(0,x^1,x^2).
\end{align*}
\end{proposition}
\begin{proof}
Denote $\tilde \Phi(x)=x^1\vee x^2$ for all $x\in \R^N$ so that 
$\Phi(x)\geq \tilde \Phi(x)$. 
Now consider a second game with final condition $\tilde \Phi$.
We denote its value function by $\tilde{V}^M$. It is then clear that $V^M(m,x) \geq \tilde{V}^M(m,x)$ for all $x\in \R^N$ and $0\leq m\leq M$. 

The final condition of the second game only depends on the first two components of the state. Thus, 
$\tilde V^M(m,x)=\hat V^M(m,x^1,x^2)$ where $\hat V^M$ is the value of an auxiliary two-expert game. 
Thanks to the Proposition \ref{prop:existbalanced}, the game with two experts always admits balanced strategies. 

Thanks to Theorem \ref{thm:viscosity2}, 
\begin{align*}
\liminf\limits_{M \to \infty} \frac{\hat V^M(tM,\sqrt{M}x^1,\sqrt{M}x^2)}{\sqrt{M}} \geq \hat{u}(t,x^1,x^2).
\end{align*}
where $\hat u$ solves
\begin{align}\label{eq:pde3}
0&=\partial_t \hat u(t,x)+ \hat H_B(\nabla^2 \hat u(t,x)).
\end{align}
with final condition $\hat u(1,x)=x^1\vee x^2$ and the generator $\hat H_B$ is associated to the balanced strategies to the auxiliary two-expert game. 
Thanks to Proposition \ref{prop:special} (which will be proved independently of this Proposition), the optimizer in \eqref{eq:pde3} is associated to a constant strategy so that $\hat u$ in fact solves a linear non-degenerate PDE which concludes the proof. 

\end{proof}

\begin{remark}
Note that due to the non-degeneracy of the PDE solved by $\hat u$, we easily have that $\pa_i\hat u(t,x)>0$ for all $t\in [0,1)$. Therefore, $\hat u$ is a smooth solution of \eqref{eq:pde} and for the auxiliary two-expert game, $ \frac{1}{\sqrt{M}} \hat V^M \left(\lceil Mt \rceil, \sqrt{M}x \right)$ indeed converges to $\hat u(t,x)$. 

\end{remark}
The following Proposition shows the importance of  Assumption \ref{assume:final} (iii). For any $\Phi$ satisfying Assumption \ref{assume:final} (iii), we will show that $\Phi(x)\leq \Phi(0)+\Phi_{m,\theta}(x)$.  Therefore the forecaster is partially satisfied when he does better than the average of the experts. Since no balanced strategies exist, the forecaster can do better than the average by following the best performed expert at each round, and thus the scaled value function tends to $-\infty$.

\begin{proposition}\label{prop:emptycase2}
Assume that the terminal condition $\Phi$ satisfies Assumption \ref{assume:final}. Then, we obtain that 
\begin{align*}
\limsup \limits_{M \to \infty} \frac{1}{\sqrt{M}}V^M(0,\sqrt{M} x) = -\infty. 
\end{align*}
\end{proposition}
\begin{proof}
Recall that $\Phi_m(x)=\max_i x^i$.
For any $x \in \mathbb{R}^N$, due to Assumption \ref{assume:final} (i), (ii), it follows that 
\begin{align*}
\Phi(x) \leq \Phi( \Phi_m(x) \1)= \Phi(0)+ \Phi_m(x). 
\end{align*}
And by Assumption \ref{assume:final} (iii), we obtain that 
\begin{align*}
\Phi(0)+ \Phi_m(x)-\Phi(x)=\Phi(\Phi_m(x) \1)-\Phi(x) \geq \frac{\theta}{N} (\Phi_m(x)\1 - x) \cdot \1,
\end{align*}
and therefore
\begin{align}\label{eq:inequalityg}
\Phi(x) & \leq \Phi(0)+\Phi_m(x)- \frac{\theta}{N} (\Phi_m(x)\1 - x) \cdot \1\notag \\
& =\Phi(0)+(1-\theta) \Phi_m(x)+ \frac{\theta}{N} x \cdot \1  =: \Phi(0)+\Phi_{m,\theta}(x). 
\end{align}

Since $V^M(0, \sqrt{M}x) \leq V^M(0,0)+\sqrt{M}  \Phi(x)$, it suffices to prove that 
\begin{align*}
\limsup \limits_{M \to \infty} \frac{1}{\sqrt{M}}V^M(0,0) = -\infty. 
\end{align*}
Denote $\boldsymbol{\phi}:= (\phi_1, \dotso, \phi_{M})$ the sequence of strategies of the forecaster, and $\boldsymbol{\a}:=(\a_1,\dotso, \a_{M})$ the sequence of strategies of the adversary.  Due to \eqref{eq:inequalityg}, we obtain that
\begin{align}
V^M(0,0)&= \inf_{\boldsymbol{\phi}} \sup_{\boldsymbol{\a}} \mathbb{E}^{\boldsymbol{\phi},\boldsymbol{\a}} [\Phi(X_M) \, | \, X_0=0]  \notag \\
& \leq \Phi(0)+ \inf_{\boldsymbol{\phi}} \sup_{\boldsymbol{\a}} \mathbb{E}^{\boldsymbol{\phi},\boldsymbol{\a}}[\Phi_{m,\theta}(X_M) \, | \, X_0=0].
\end{align}
For any $\a \in \mathcal{A}$, we define $$M(\alpha)= \max_{i} \mathbb{E}^{\a} [\Delta G^i],$$ and $$m(\alpha)=\max_{i}\{  \mathbb{E}^{\a}[ \Delta G^i]:  \mathbb{E}^{\a}[ \Delta G^i]< M(\alpha)\}.$$
Here $M(\a)$ is the largest expected expert gain under the policy $\a$, and $m(\a)$ is the second largest expected gain. Since $\mathcal{A}_B =\emptyset$, for any $\a \in \mathcal{A}$ we have that $m(\alpha) \geq 0$ and $M(\alpha)-m(\alpha) >0$. Define 
\begin{align*}
\delta:= \inf_{a \in \mathcal{A}} (M(\a)-m(\a)).
\end{align*}
It can be easily seen that $\delta >0$. 

For any value function $V$ \eqref{def:V} with terminal condition $\Phi$ \eqref{def:terminal} satisfying $\Phi(x+\lambda \1)=\Phi(x)+\lambda$, it holds that $V^M(m+1, x+ \lambda \1)= V^M(t,x)+\lambda$. It can be easily verified that
\begin{align*}
V^M(m-1,x)&=\min_{\phi_m}\max_{\a_m} \E^{\phi_m,\a_m}[V^M(m, x+\Delta X_m)] \\
&=\max_{\a_m} \E^{\a_m}[V^M(m,x+\Delta G_m)]- \max_{\phi_m} \sum_{i} \phi_m^i \E^{\a_m}[ \Delta G_m^i].
\end{align*}
Therefore for any fixed strategy $\hat{\boldsymbol{\a}}$ of the adversary, the optimal response of the forecaster is to follow the experts with maximal expected gain under policy $\hat{\boldsymbol{\a}}$ at each round, i.e., $\phi^i_m=0$ if and only if $\mathbb{E}^{\hat{\a}_m}[\Delta G^i] < M(\hat{\a}_m)$. Denote one such optimal response of the forecaster by $\hat{\bs{\phi}}$. 
Therefore we obtain that 
\begin{align}\label{eq:negativeinf1}
 \inf_{\boldsymbol{\phi}} &\mathbb{E}^{{\boldsymbol{\phi}},\hat{ \boldsymbol{\a}}}[\Phi_{m,\theta}(X_M) \, | \, X_0=0] = \mathbb{E}^{\hat{\boldsymbol{\phi}},\hat{ \boldsymbol{\a}}}[\Phi_{m,\theta}(X_M) \, | \, X_0=0] \notag\\
 & = (1-\theta)  \mathbb{E}^{\hat{\boldsymbol{\phi}},\hat{ \boldsymbol{\a}}}[\Phi_m(X_M) \, | \, X_0=0]+\frac{\theta}{N} \mathbb{E}^{\hat{\boldsymbol{\phi}},\hat{ \boldsymbol{\a}}}[ {X}_M \cdot \1  \, | \, X_0=0] \notag \\
& = (1-\theta)  \inf_{\boldsymbol{\phi}} \mathbb{E}^{\boldsymbol{\phi},\hat{ \boldsymbol{\a}}}[\Phi_m(X_M) \, | \, X_0=0]+\frac{\theta}{N} \inf_{\boldsymbol{\phi}} \mathbb{E}^{\boldsymbol{\phi},\hat{ \boldsymbol{\a}}}[ {X}_M \cdot \1  \, | \, X_0=0].
\end{align}
According to Remark~\ref{rmk:upperbound}, there exists some positive $C$ independent of choice of  $\hat{\bs{\a}}$ such that
\begin{align}\label{eq:negativeinf2}
\limsup\limits_{M \to \infty} \frac{1}{\sqrt{M}}\inf_{\boldsymbol{\phi}} \mathbb{E}^{\boldsymbol{\phi},\hat{\boldsymbol{\a}}}[\Phi(X_M) \, | \, X_0=0] \leq C. 
\end{align}
Due to our definition of $\delta$, we obtain that for any $\hat{\boldsymbol{\a}}$
\begin{align}\label{eq:negativeinf3}
\inf_{\boldsymbol{\phi}} \mathbb{E}^{\boldsymbol{\phi},\hat{\boldsymbol{\a}}}[ {X}_M \cdot \1/N \, | \, X_0=0] \leq -\delta M/N. 
\end{align}
In conjunction with \eqref{eq:negativeinf1},\eqref{eq:negativeinf2} and \eqref{eq:negativeinf3}, we conclude that 
\begin{align*}
\limsup \limits_{M \to \infty} \frac{1}{\sqrt{M}}V^M(0,0) = -\infty. 
\end{align*}

\end{proof}

\section{Explicit solutions in some special cases}\label{sec:special}
In this section we exhibit some cases where the value function and the strategies of the adversary can be explicitly computed. The results are valid for final condition 
$\Phi_{\theta,m}(x)=(1-\theta)\Phi(x)+\frac{\theta}{N} \sum_{i=1}^N x^i $ for any fixed $\theta \geq 0$. 

Our methodology is to provide a stochastic representation for the solution of \eqref{eq:pde2} that allows us to claim that this solution also solves \eqref{eq:pde}. Then, we use this solution to obtain a strategy for the adversary. 
\begin{definition}
Any $\bar{\alpha}\in \mathcal{A}_B$ satisfying 
$$c_{\bar{\alpha}}=\sup_{\alpha\in \mathcal{A}_B}c_\alpha$$ is called a generous adversary and 
any $\underline{\alpha}\in \mathcal{A}_B$ satisfying 
$$c_{\underline{\alpha}}=\inf_{\alpha\in \mathcal{A}_B}c_\alpha$$ is called a greedy adversary. 
\end{definition}
Note that for $\a \in \cA_B$, one has 
\begin{align}\label{eq:cont1}
\sum_{i,j=1}^N S_{ij} \E^{\a}[\Delta G^i \Delta G^j]=c_\alpha Tr\left(\Sigma_1 S\right)-Tr\left(\Sigma_2 S\right)
\end{align}
where
$\Sigma_1=\{\mu^i+\mu^j\}_{i,j=1}^N+diag(1-2\mu^1,\ldots,1-2\mu^N)$ and $\Sigma_2=\{ \mathbbm{1}_{\{i \not = j\}}\mu^i\mu^j \}_{i,j=1}^N$. 

The next lemma shows that the linear differential operator associated to each balanced strategy is non-degenerate. 

\begin{lemma}\label{lem:posdef}
If $N\geq 2$, and $ 0<\mu^i <1, i =1, \dotso, N$, then for any $\a \in \mathcal{A}_B$ (if this set is not empty)  the matrix $c_{\a} \Sigma_1- \Sigma_2$ is positive definite.
\end{lemma}
\begin{proof}
It suffices to show that for any vector $y=(y^1, \dotso, y^N) \not = 0$,
\begin{align}\label{eq:posdef}
y (c_{\a}\Sigma_1-\Sigma_2) y^\top= \sum_{i,j=1}^N y^iy^j \E^{\a}[\Delta G^i \Delta G^j]=\E^{\a}\left[\left(y^\top \Delta G \right)^2 \right]>0.
\end{align}
As a result of \eqref{eq:posdef}, $y (c_{\a}\Sigma_1-\Sigma_2) y^\top=0$ if and only if $y^\top \Delta G=0$\, $\mathbb{P}^{\a}$-a.s. Denote the collection of all the possible realizations of $\Delta G$ by 
\begin{align*}
O=\{ z \in \{0,1\}^N: \, \mathbb{P}^{\a}[\Delta G=z]>0 \}.
\end{align*}
We will prove that  the dimension of the linear expansion $\langle  O \rangle$ is $N$. Then it follows that $y^\top \Delta G=0$\, $\mathbb{P}^{\a}$-a.s. is impossible, and hence $ y (c_{\a}\Sigma_1-\Sigma_2) y^\top>0$. 

Recall that $b^i$ is the probability that the adversary corrupts expert $i$ and sets his gain to $1$. Suppose there exists some $b^i >0$. For any $z \in \{0,1\}^N$, we define $ \hat{z}= z- z^i e^i+e^i$, and $p_j=\mathbbm{1}_{\{z^j=0\}}(1-\mu^j)+\mathbbm{1}_{\{z^j=1\}} \mu^j $ for any $j \not = i$. Note that the $i$-th coordinate of $\hat{z}$ is always $1$. It can be easily seen that 
\begin{align*}
\P(\Delta G= \hat{z}) & \geq \P(\Delta G=\hat{z}, \text{the adversary corrupts expert $i$ and sets his gain to $1$})\\
& \geq b^i \prod_{j=1, j \not = i}^N p_i >0.
\end{align*}
Therefore $\hat{z} \in O$ for any $z$, and $O \supset \{ z \in \{0,1\}^N: z^i=1\}$. Hence the dimension of $\langle O \rangle$ is $N$. 

Recall that $a^i$ is the probability that the adversary corrupts expert $i$ and sets his gain to $0$. If $b^i=0, i=1, \dotso, N$, there must exist some $a^i>0$. For any $z \in \{0,1\}^N$, we define $\tilde{z}=z-z^i e^i $. Note that the $i$-coordinate of $\tilde{z}$ is always $0$. Since $\P(\Delta G= \tilde{z}) \geq a^i \prod_{j=1, j \not= i}^N p_i>0$, we obtain that $\tilde{z} \in O$ for any $z$, and hence $O \supset \{z \in \{0,1\}^N: z^i=0 \}$. Since $N c_{\a} = \sum_{j=1}^N((1-a^j-b^j)\mu^j+ b^j)>0$ for $N \geq 2$, it must hold that $\E^{\a}[\Delta G^i]>0$. There must exist some $z \in O$ such that $z^i=1$. Hence $\langle O \rangle $ is of dimension $N$.
\end{proof}

The next proposition provides a condition on $\{\mu^i\}$ that allows us to characterize the optimal strategy of the adversary as a constant strategy. Therefore, in this case, the optimal strategy of the adversary is the maximizer of the Hamiltonian in \eqref{eq:pde2}.

\begin{proposition}\label{prop:special}
If \eqref{eq:existbalanced} holds, $\mu^i \in (0,1)$ for any $i$  and $\mu^i+\mu^j \leq 1$ (resp. $\mu^i+\mu^j \geq 1$) for any $i \not = j$, then any generous adversary $\overline \a$ (resp. greedy adversary $\underline \a$) is the maximizer of the Hamiltonian in \eqref{eq:pde} for all $(t,x)\in [0,1)\times \R^N$.

Additionally, the solution $u$ to \eqref{eq:pde}  is equal to $U$ (which is the solution to \eqref{eq:pde2}),  and satisfies \eqref{eq:boundder}. Therefore according to Proposition~\ref{prop:asymptotic}, the asymptotic strategy $\bs{\a}^*$ in \eqref{eq:asymptoticadversary} for the adversary guarantees $U$ as a lower bound of regret. 
\end{proposition}

\begin{proof}
Note that according to Proposition~\ref{prop:existbalanced}, $\mathcal{A}_B \not = \emptyset$ if and only if \eqref{eq:existbalanced} holds. We only provide the proof for the case where $\mu^i+\mu^j \geq 1$ for any $i \not = j$. The other case can be proved similarly. Our methodology is to show that the solution of the linear PDE associated with the constant strategy $\underline \a$ also provides a solution to \eqref{eq:pde2} and \eqref{eq:pde}.

{\it Step 1: Approximating the final condition: } Using the definition of $\Phi_{m,\theta}$ and the assumption on $\{\mu^i\}$, we first find an approximation sequence $\Phi_{\epsilon}$ such that $\Phi_{\epsilon}$ converges to $\Phi_{m,\theta}$ in $\mathcal{L}^{\infty}$ as $\epsilon \to 0$ and $Tr(\Sigma_1  \pa^2 \Phi_{\epsilon}(x) ) \leq 0$. Since
$$\Phi_{m,\theta}(x)=(1-\theta)\max\{x^1,\dotso,x^N\}+\frac{\theta}{N} \sum_{i=1}^N x^i=(1-\theta)\Phi_{m}(x)+\frac{\theta}{N} \sum_{i=1}^N x^i $$ 
and the second derivative of the linear part is $0$, it is sufficient to prove the claim for $\theta=0$. 
We prove the existence of such $\Phi_{\epsilon}$ by induction. 

First we approximate the absolute value function on $\mathbb{R}^1$. For each $\epsilon >0$, it can be easily verified that there exists some $f_{\epsilon}:\mathbb{R}^1 \to \mathbb{R}^1$ such that 
\begin{enumerate}[(i)]
\item $f_{\epsilon}(x)=|x|$ if $|x| \geq \epsilon$;
\item $f_{\epsilon}$ is convex;
\item $|x| \leq f_{\epsilon} \leq |x|+\epsilon, \, \forall x \in \mathbb{R}^1$.
\end{enumerate}

Then in the case of $N=2$, we define 
\begin{align*}
\Phi^2_{\epsilon}(x^1,x^2):= \frac{x^1+x^2+f_{\epsilon}(x^1-x^2) }{2}.
\end{align*}
It can be easily seen that $\Phi^2_{\epsilon}$ converges to $\Phi^2$ in $\mathcal{L}^{\infty}$, and $\pa_1 \Phi^2_{\e}+\pa_2 \Phi^2_{\e}=1$. We compute the second derivative of $\Phi^2_{\epsilon}$ and obtain that 
\begin{align*}
&\pa^2_{11} \Phi^2_{\e}(x)= \pa^2_{22} \Phi^2_{\e}(x)=\frac{1}{2} f_{\e}''(x^1-x^2), \\
& \pa^2_{12} \Phi^2_{\e}(x)= \pa^2_{12} \Phi^2_{\e}(x)=-\frac{1}{2} f_{\e}''(x^1-x^2).
\end{align*}
Since $f_{\e}$ is convex, we have $f_{\e}''(x^1-x^2) \geq 0$, and therefore 
\begin{align*}
Tr(\Sigma_1 \pa^2 \Phi^2_{\e}(x))= (\mu^1+\mu^2 -1) f''_{\e}(x^1-x^2) \leq 0. 
\end{align*}

Suppose our claim is correct for $N-1$ many experts, let us prove it for $N$. Without loss of generality, we assume that $\mu^N= \max\{\mu^1, \dotso, \mu^N\}$. Denote by $\tilde{x}$ the first $N-1$ components of $x$, and by $\tilde{\Sigma}_1$ the principal submatrix of $\Sigma_1$ by removing its last row and column. By induction, we have $\Phi^{N-1}_{\e}$ such that 
\begin{enumerate}[(i)]
\item $\sum_{i=1}^{N-1} \pa_{i} \Phi^{N-1}_{\e}=1$ and $\pa_i \Phi^{N-1}_{\e} \geq 0, \, \forall i \leq N-1$; 
\item $\Phi^{N-1}_{\e} \to \Phi^{N-1}$ in $\mathcal{L}^{\infty}$ as $\e \to 0$;
\item $Tr(\tilde{\Sigma}_1 \pa^2 \Phi^{N-1}_{\e}) \leq 0$. 
\end{enumerate}
Define 
\begin{align*}
\Phi_{\e}(x):=\frac{\Phi^{N-1}_{\e}(\tilde{x})+x^N+f_{\e}(\Phi^{N-1}_{\e}(\tilde{x})-x^N )}{2}.
\end{align*}
It is then clear that $\Phi_{\e} \to \Phi$ in $\mathcal{L}^{\infty}$. To simplify notation, we omit the arguments $x, \tilde{x}$ when it is clear from the context. Let us compute its first derivatives 
\begin{align*}
&\pa_i \Phi_{\e} = \frac{1}{2} \pa_i \Phi^{N-1}_{\e}+ \frac{1}{2} f'_{\e}(\Phi^{N-1}_{\e}-x^N) \pa_i \Phi^{N-1}_{\e}, \ i \leq N-1, \\
& \pa_N \Phi_{\e} =\frac{1}{2} - \frac{1}{2}f'_{\e}(\Phi^{N-1}_{\e}-x^N), 
\end{align*}
and  second derivatives
\begin{align*}
&\pa^2_{ij} \Phi_{\e}=\frac{1}{2} \left(1 +f'_{\e}(\Phi^{N-1}_{\e}-x^N) \right) \pa^2_{ij} \Phi^{N-1}_{\e}+ \frac{1}{2} f''_{\e}(\Phi^{N-1}_{\e}-x^N) \pa_i \Phi^{N-1}_{\e} \pa_j \Phi^{N-1}_{\e}, \ i,j \leq N-1, \\
& \pa^2_{iN} \Phi_{\e}= -\frac{1}{2} f''_{\e}(\Phi^{N-1}_{\e}-x^N) \pa_i \Phi^{N-1}_{\e}, \ i \leq N, \\
& \pa^2_{NN} \Phi_{\e}=\frac{1}{2}f''_{\e}(\Phi^{N-1}_{\e}-x^N). 
\end{align*}

Due to  $\sum_{i=1}^{N-1} \pa_{i} \Phi^{N-1}_{\e}=1$ and $1+f_{\e}',1-f_{\e}' \geq 0$, we obtain that $\pa_i \Phi_{\e} \geq 0$ and 
\begin{align}\label{eq:induction1}
\sum_{i=1}^{N} \pa_{i} \Phi^{N}_{\e}=1.
\end{align}
Denote by $\widetilde{\pa^2 \Phi^{N}_{\e} }$ the principal submatrix of $\pa^2 \Phi^{N}_{\e}$ by removing the last row and column. We rewrite the trace as 
\begin{align*}
Tr( \Sigma_1 \pa^2 \Phi^{N}_{\e}) =& Tr( \tilde{\Sigma}_1 \widetilde{\pa^2 \Phi^{N}_{\e} })+ 2 \sum_{i=1}^{N-1} (\mu^i+\mu^N) \pa^2_{iN} \Phi_{\e} + \pa^2_{NN} \Phi_{\e} \\
 =& \frac{1}{2}\left(1 +f'_{\e}(\Phi^{N-1}_{\e}-x^N) \right) Tr(\tilde{\Sigma}_1 \pa^2 \Phi^{N-1}_{\e}) \\
+&\frac{1}{2} f_{\e}''(\Phi^{N-1}_{\e}-x^N) \left((\pa \Phi^{N-1}_{\e})^\top  \cdot \tilde{\Sigma}_1 \cdot \pa \Phi^{N-1}_{\e}-2 \sum_{i=1}^{N-1}(\mu^i+\mu^N) \pa_i \Phi^{N-1}_{\e} +1\right).
\end{align*}
According to our induction, we know that $Tr(\tilde{\Sigma}_1 \pa^2 \Phi^{N-1}_{\e}) \leq 0$. Since $\left(1 +f'_{\e}(\Phi^{N-1}_{\e}-x^N) \right) \geq 0, \, f_{\e}''(\Phi^{N-1}_{\e}-x^N) \geq 0$, it suffices to show that 
\begin{align}\label{eq:induction2}
 \left((\pa \Phi^{N-1}_{\e})^\top  \cdot \tilde{\Sigma}_1 \cdot \pa \Phi^{N-1}_{\e}-2 \sum_{i=1}^{N-1}(\mu^i+\mu^N) \pa_i \Phi^{N-1}_{\e} +1\right) \leq 0.
\end{align}

Due to the equality $\sum_{i=1}^{N-1} \pa_{i} \Phi^{N-1}_{\e}=1$, we obtain that 
\begin{align*}
(\pa \Phi^{N-1}_{\e})^\top  \cdot \tilde{\Sigma}_1 \cdot \pa \Phi^{N-1}_{\e}=&(\sum_{i=1}^{N-1} \pa_{i} \Phi^{N-1}_{\e} )^2+ \sum_{i \not = j \leq N-1 } (\mu^i+\mu^j-1) \pa_i \Phi^{N-1}_{\e}  \pa_j \Phi^{N-1}_{\e} \\
=&1+\sum_{i \not = j \leq N-1 } (\mu^i+\mu^j-1) \pa_i \Phi^{N-1}_{\e}  \pa_j \Phi^{N-1}_{\e}. 
\end{align*}
Similarly, we have that 
\begin{align*}
2 \sum_{i=1}^{N-1}(\mu^i+\mu^N) \pa_i \Phi^{N-1}_{\e}=2+\sum_{i=1}^{N-1}(\mu^i+\mu^N-1) \pa_i \Phi^{N-1}_{\e}. 
\end{align*}
Therefore \eqref{eq:induction2} is equivalent to that 
\begin{align}\label{eq:induction3}
\sum_{i \not = j \leq N-1 } (\mu^i+\mu^j-1) \pa_i \Phi^{N-1}_{\e}  \pa_j \Phi^{N-1}_{\e} \leq \sum_{i=1}^{N-1}(\mu^i+\mu^N-1) \pa_i \Phi^{N-1}_{\e}.
\end{align}
For fixed $i \leq N-1$, according to our assumption $\mu^N=\max\{\mu^1, \dotso, \mu^N\}$ we have that 
\begin{align*}
\sum_{j \leq N-1, j \not = i} (\mu^i+\mu^j-1)\pa_i \Phi^{N-1}_{\e}  \pa_j \Phi^{N-1}_{\e} & \leq  (\mu^i+\mu^N-1) \pa_i \Phi^{N-1}_{\e} \sum_{j \leq N-1, j \not = i} \pa_j \Phi^{N-1}_{\e} \\
& \leq (\mu^i+\mu^N-1) \pa_i \Phi^{N-1}_{\e}. 
\end{align*}
Summing from $i=1$ to $i=N-1$, we obtain the inequality \eqref{eq:induction3}, and hence \eqref{eq:induction2}. In conjunction with \eqref{eq:induction1}, we finish proving the induction.

{\it Step 2: Solving the nonlinear PDE with the linear PDE:} Denote by $u$ the solution of 
$$0=\partial_t u(t,x) +\frac{1}{2} Tr\left(\underline \Sigma \pa_{xx}^2 u(t,x)\right)$$ with terminal condition $\Phi_{\theta,m} $, and by $ u^\e$ the solution of
\begin{align}\label{eq:pdeapprox}
0=\partial_t u^\e(t,x) +\frac{1}{2} Tr\left(\underline \Sigma \pa_{xx}^2 u^\e(t,x)\right)
\end{align}
with terminal condition $\Phi_{\e} $(smooth approximation of $\Phi_{\theta,m} $ satisfying $Tr(\Sigma_1  \pa^2 \Phi_{\epsilon}(x) ) \leq 0$). In order to show that $\underline{\alpha}$ is optimal, it suffices to prove that $u$ solves PDE \eqref{eq:pde}. To show this solution property, it is sufficient to show that $u$ solves \eqref{eq:pde2},
$\pa_i u(t,x)>0$ for all $(t,x)\in [0,1)\times \R^N$, and $Tr(\Sigma_1 \pa_{xx}^2 u(t,x))\leq 0$.

We can differentiate \eqref{eq:pdeapprox} in $x$ twice to obtain that
$$w^\e(t,x)=Tr(\Sigma_1 \pa_{xx}^2 u^\e(t,x))$$ solves the same PDE 
$$0=\partial_t w^\e(t,x) +\frac{1}{2} Tr\left(\underline \Sigma \pa_{xx}^2 w^\e(t,x)\right)$$
with final condition 
$$w^\e(1,x)=Tr(\Sigma_1  \pa_{xx}^2 \Phi_\e(x)).$$
Due to the choice of $\Phi_\e$ we have that that $w^\e(1,x)\leq 0$. 

Thus, by the maximum principle, $w^\e(t,x)\leq 0$ for all $(\e,t,x)\in (0,1)\times [0,1]\times \R^N$. 
Fix $t\in [0,1)$. Due the Malliavin calculus representation of $\pa_{xx}^2 v^\e$,
$$\pa_{xx}^2 u^\e(t,x)=\E\left[\Phi_\e(x+\sqrt{\underline \Sigma} (W_{1-t}))\sqrt{\underline \Sigma}^{-1}\frac{W_{1-t}W^\top_{1-t}-(1-t)I_N}{(1-t)^2} \sqrt{\underline \Sigma}^{-1}\right]$$
and we have that $\pa_{xx}^2 u^\e(t,x)\to \pa_{xx}^2 u(t,x).$ This implies that for all $t\in [0,1) $ and $x\in \R^N$, 
$$Tr(\Sigma_1 \pa_{xx}^2 u(t,x))\leq 0.$$
Thus, thanks to \eqref{eq:cont1}
\begin{align*}
2H_B( \pa_{xx}^2 u(t,x))&=\sup_{\a\in \cA_B}\sum_{i,j=1}^N \pa_{ij}^2 u(t,x) \E^{\a}[\Delta G^i \Delta G^j] \\
&=\sup_{\a\in \cA_B}c_\alpha Tr\left(\Sigma_1  \pa_{xx}^2 u(t,x)\right)-Tr\left(\Sigma_2  \pa_{xx}^2 u(t,x)\right)\\
&=c_{\underline\alpha} Tr\left(\Sigma_1  \pa_{xx}^2 u(t,x)\right)-Tr\left(\Sigma_2  \pa_{xx}^2 u(t,x)\right),
\end{align*}
and therefore $\underline \a$ is optimal among balanced strategies and $u$ solves \eqref{eq:pde2}  and is therefore equal to $U$.

Note also that using the the density of the Brownian motion one can show that 
\begin{align}\label{eq:derivsol}
\pa_{i}u\left(t,x\right)=(1-\theta)\P\left(i\mbox{th coordinate of }(x+\sqrt{\underline \Sigma} W_{1-t}) \mbox{ is maximal}\right)+\frac{\theta}{N}>0
\end{align}
for all $(t,x)\in [0,1)\times \R^N$ and $\theta \geq 0$. Thus, $u$ also solves \eqref{eq:pde}.

{\it Step 3: The derivatives of $u$ satisfies \eqref{eq:boundder}.}
According to Lemma~\ref{lem:posdef} and Proposition~\ref{prop:special}, the coefficient matrix $\Sigma$ of the optimal adversary is positive definite. Then there exists some matrix $P=(P_1, \dotso, P_N)$ with $\det P>0$ such that $\underline\Sigma=P^\top P$. It can be easily verified that the solution is given by 
\begin{align*}
u(1-t,x)=\frac{\det P} { (2\pi t)^{N/2}} \int_{\mathbb{R}^N} e^{-\frac{|Py|^2}{2t} } {\Phi}_{m,\theta}(x-y) \, dy.
\end{align*}
Note that ${\Phi}_{m,\theta}(x)= (1-\theta) \max\{x^1,\dotso,x^N\}+\frac{\theta}{N} \sum x^i$ is differentiable almost everywhere and 
\begin{align*}
\pa_{i} {\Phi}_{m,\theta} (x-y)=(1-\theta) \mathbbm{1}_{\{x^i-y^i \geq \Phi_m(x-y) \} } +\frac{\theta}{N}\quad  \text{ a.e.}
\end{align*}
Therefore we obtain that 
\begin{align}
\pa_i u(1-t,x)& = \frac{\theta}{N}+\frac{(1-\theta)\det P}{(2\pi t)^{N/2}} \int_{\mathbb{R}^N} e^{-\frac{|Py|^2}{2t} }  \mathbbm{1}_{\{x^i-y^i \geq {\Phi}_m(x-y) \} } \, dy \notag \\
& = \frac{\theta}{N}+\frac{(1-\theta)\det P}{(2\pi t)^{N/2}} \int_{\mathbb{R}^N} e^{-\frac{|Px-Py|^2}{2t} }  \mathbbm{1}_{\{y^i \geq \Phi_m(y) \} } \, dy.
\end{align}
Differentiating the above equation with respect to $x^j$, it follows that 
\begin{align}
\pa^2_{ij} u(1-t,x)= -\frac{(1-\theta)\det P}{(2\pi t)^{N/2}} \int_{\mathbb{R}^N} e^{-\frac{|Px-Py|^2}{2t} } \left(\frac{P_j^\top P (x-y)}{t} \right) \mathbbm{1}_{\{y^i \geq \Phi_m(y) \} } \, dy.
\end{align}\label{eq:secorder}
Similarly, it can be easily seen that
\begin{align}
\pa^3_{ijk}& u(1-t,x) \notag \\
=&-\frac{(1-\theta)\det P}{(2\pi t)^{N/2}} \int_{\mathbb{R}^N} e^{-\frac{|Px-Py|^2}{2t} } \left(\frac{P_j^\top P_k}{t} \right) \mathbbm{1}_{\{y^i \geq \Phi_m(y) \} } \, dy \notag \\
&+\frac{(1-\theta)\det P}{(2\pi t)^{N/2}} \int_{\mathbb{R}^N} e^{-\frac{|Px-Py|^2}{2t} } \left(\frac{P_j^\top P(x-y)}{t} \right)\left(\frac{P_k^\top P(x-y)}{t} \right)\mathbbm{1}_{\{y^i \geq \Phi_m(y) \} } \, dy, \notag
\end{align}
and also
\begin{align} 
\pa^4_{ijkl}& u(1-t,x)  \notag \\
=&\frac{(1-\theta)\det P}{(2\pi t)^{N/2}} \int_{\mathbb{R}^N} e^{-\frac{|Px-Py|^2}{2t} } \left(\frac{P_j^\top P_k}{t} \right)\left(\frac{P_l^\top P (x-y)}{t} \right) \mathbbm{1}_{\{y^i \geq \Phi_m(y) \} } \, dy \notag \\
&+\frac{(1-\theta)\det P}{(2\pi t)^{N/2}} \int_{\mathbb{R}^N} e^{-\frac{|Px-Py|^2}{2t} } \left(\frac{P_j^\top P(x-y)}{t} \right)\left(\frac{P_k^\top P_l}{t} \right) \mathbbm{1}_{\{y^i \geq \Phi_m(y) \} } \, dy \notag \\
&+\frac{(1-\theta)\det P}{(2\pi t)^{N/2}} \int_{\mathbb{R}^N} e^{-\frac{|Px-Py|^2}{2t} } \left(\frac{P_j^\top P_l}{t} \right)\left(\frac{P_k^\top P(x-y)}{t} \right) \mathbbm{1}_{\{y^i \geq \Phi_m(y) \} } \, dy \notag \\
&-\frac{(1-\theta)\det P}{(2\pi t)^{N/2}} \int_{\mathbb{R}^N} e^{-\frac{|Px-Py|^2}{2t} } \left(\frac{P_j^\top P(x-y)}{t} \right)\left(\frac{P_k^\top P(x-y)}{t} \right)\left(\frac{P_l^\top P (x-y)}{t} \right) \mathbbm{1}_{\{y^i \geq \Phi_m(y) \} } \, dy. \notag
\end{align}

Let us show that there exists a constant $C$ such that for any $x \in\mathbb{R}^N$
\begin{align*}
|\pa_{ij}^2 u(1-t,x) | \leq \frac{C}{\sqrt{t}}.
\end{align*}
It can be easily seen that 
\begin{align*}
|\pa^2_{ij} u(1-t,x)|=& \left| -\frac{(1-\theta)\det P}{(2\pi t)^{N/2}} \int_{\mathbb{R}^N} e^{-\frac{|Py|^2}{2t} } \left(\frac{P_j^\top P y}{t} \right) \mathbbm{1}_{\{x^i-y^i \geq \Phi_m(x-y) \} } \, dy \right| \\
= & \left| -\frac{(1-\theta)\det P}{(2\pi)^{N/2}} \int_{\mathbb{R}^N} e^{-\frac{|Py|^2}{2} } \left(\frac{P_j^\top P y}{\sqrt{t}} \right) \mathbbm{1}_{\left\{\frac{x^i}{\sqrt{t}}-y^i \geq \Phi_m\left(\frac{x}{\sqrt{t}}-y\right) \right\} } \, dy \right| \\
\leq & \frac{(1-\theta)\det P}{(2\pi)^{N/2} \sqrt{t}} \int_{\mathbb{R}^N}e^{-\frac{|Py|^2}{2} } |P_j^\top P y| \, dy =\frac{C}{\sqrt{t}}.
\end{align*}
Similarly, it can be proved that 
\begin{align*}
|\pa^3_{ijk} u(1-t,x)| \leq \frac{C}{t},
\end{align*}
and 
\begin{align*}
|\pa^4_{ijkl} u(1-t,x)| \leq \frac{C}{t^{\frac{3}{2}}}.
\end{align*}

Now let us estimate $\pa^2_{tt}u$ and $\pa_{tx}^2 u$. Since $\pa_t u= - \frac{1}{2} Tr( \Sigma \pa_{xx}^2 u )$, we obtain that 
\begin{align*}
|\pa_{tt}^2 u (1-t,x)|&=\left|\frac{1}{2} \pa_t Tr( \Sigma \pa_{xx}^2 u(1-t,x) )\right| =\left|\frac{1}{2} Tr(\Sigma \pa_{xx}^2 \pa_{t} u(1-t,x) ) \right| \\
&= \left|\frac{1}{2} Tr\left(\Sigma \pa_{xx}^2 \left(\frac{1}{2} Tr(\Sigma \pa^2_{xx} u(1-t,x)) \right) \right) \right|.
\end{align*}
The right hand of the above equation is a linear combination of $\pa_{ijkl}^4 u(1-t,x)$, and hence 
\begin{align*}
|\pa_{tt}^2 u (1-t,x)| \leq \frac{C}{{t}^{\frac{3}{2}}}, \quad \text{ $\forall x \in \mathbb{R}^N$.}
\end{align*}
Similarly, it can be easily verified that 
\begin{align*}
|\pa_{xxx}^3 u(1-t,x)| \leq \frac{C}{t}, \quad |\pa_{tx}^2 u (1-t,x)| \leq \frac{C}{t}, \quad \text{ $\forall x \in \mathbb{R}^N$},
\end{align*}
which concludes the proof of \eqref{eq:boundder}.

\end{proof}

\begin{example}
In the special case of $\mu^i=\mu, i=1, \dotso,N$, the equation becomes 
\begin{align*}
0=\partial_t U(t,x)+ \frac{1}{2} \max_{\alpha\in \mathcal{A}_B}\left( \sum_{i,j=1}^N (2\mu c_{\alpha}-\mu^2) \pa_{ij}^2 U(t,x)  + \sum_{i=1}^N (c_{\alpha}- 2\mu c_{\alpha}+\mu^2) \pa^2_{ii} U(t,x)\right).
\end{align*}
Since $\1 \cdot \nabla U=1$, we have $\sum_{i,j=1}^N (2\mu c_{\alpha}-\mu^2) \pa_{ij}^2 U(t,x)=0.$ Hence the equation can be simplified as 
\begin{align}\label{eq:laplace}
0=\partial_t U(t,x)+ \frac{1}{2} \max_{\alpha\in \mathcal{A}_B}  \left( (c_{\alpha}- 2\mu c_{\alpha}+\mu^2) \Delta U(t,x) \right).
\end{align}
Note that 
\[
C:=\frac{1}{2}\max_{\alpha\in \mathcal{A}_B}  (c_{\alpha}- 2\mu c_{\alpha}+\mu^2)=
\begin{cases}
\frac{1}{2}(1-2\mu)(\mu+\frac{1}{N}(1-\mu))+\frac{1}{2}\mu^2 & \text{ if $\mu \leq 1/2$ }, \\
\frac{1}{2}(1-2\mu)(\mu-\frac{1}{N}\mu)+\frac{1}{2}\mu^2 & \text{ if $\mu \geq 1/2$ }.
\end{cases}
\]
 It can be easily verified that $C \geq 0$. We claim that the solution of \eqref{eq:laplace} is the solution of the following equation, 
\begin{align}\label{eq:simplesol}
0=\partial_t U(t,x)+ C \Delta U(t,x),
\end{align}
additionally the asymptotic strategy $\bs{\a}^*$ in \eqref{eq:asymptoticadversary} for the adversary guarantees $U$ as a lower bound of regret. 
\end{example}

\section{Conclusion}
{In this paper, we study an expert prediction problem, where an adversary only corrupts one expert at each round and a forecaster makes predictions based on experts' past gains. The forecaster aims at minimizing his regret, while the adversary wants to maximize it. Therefore this problem can be interpreted as a zero-sum game between the adversary and the forecaster.  Using viscosity theory tools in the field of partial differential equation, we provided the growth rate of regret for the forecaster. A strategy of the adversary is called balanced if the expected gain of each expert is the same under this strategy.  We showed that the growth rate of regret  fundamentally depends on whether balanced strategies  exist and whether the final condition $\Phi$ satisfies the strictly monotone condition Assumption~\ref{assume:final} (iii),
\begin{enumerate}[(i)]
\item \emph{Balanced strategies exist, $\Phi$ does not satisfy Assumption~\ref{assume:final} (iii):} the growth rate of regret is bounded below by the solution of \eqref{eq:pde2}; see Theorem~\ref{thm:viscosity2}. 
\item \emph{Balanced strategies exist, $\Phi$ satisfies Assumption~\ref{assume:final} (iii):} the growth rate of regret is given by the solution of \eqref{eq:pde2}; see Theorem~\ref{thm:viscosity1}. 
\item \emph{Balanced strategies do not exist, $\Phi$ does not satisfy Assumption~\ref{assume:final} (iii):} the asymptotic regret is of order $\sqrt{M}$; see Proposition~\ref{prop:emptycase1}. 
\item \emph{Balanced strategies do not exist, $\Phi$ satisfies Assumption~\ref{assume:final} (iii):} the asymptotic regret is $-\infty$; see Proposition~\ref{prop:emptycase2}.
\end{enumerate}
Also, we designed asymptotic optimal strategies for the adversary in Proposition~\ref{prop:asymptotic}, and solved \eqref{eq:pde}, \eqref{eq:pde2} in some special cases; see Proposition~\ref{prop:special}. 

}

\acks{E. Bayraktar is supported in part by the National Science Foundation and in part by the Susan M. Smith Chair. I. Ekren is supported in part by NSF Grant DMS 2007826. }


\appendix
\section*{Appendix A.}
\begin{proof}[Proof of Proposition \ref{prop:existbalanced}]
Suppose $\mathcal{A}_B \not = \emptyset$, and $\alpha \in \mathcal{A}_B$. Then according to \eqref{def:c}, we have that $c_{\alpha}= (1-a^j-b^j)\mu^j+b^j, \, \forall j $. If $c_{\alpha} \geq \mu^j$, we have that $$c_{\alpha}-\mu^j= -(a^j+b^j) \mu^j+b^j \leq (1-\mu^j)(a^j+b^j),$$
which is equivalent to that $ \frac{c_{\alpha}-\mu^j }{1-\mu^j} \leq a^j+b^j$. If $ \mu^j \geq c_{\alpha} $, we obtain that 
$$\mu^j-c_{\alpha}=(a^j+b^j) \mu^j- b^j \leq (a^j+b^j)\mu^j,$$
and hence $ \frac{\mu^j-c_{\alpha}}{\mu_j} \leq a^j+b^j $. Since $\sum_{j=1}^N (a^j+b^j)=1$, we get that 
\begin{align*}
\inf_{c \in [0,1]} \sum_{i=1}^N \left(\frac{\mu^i-c}{\mu^i} \vee \frac{c-\mu^i}{1-\mu^i} \right) \leq \sum_{j=1}^N  \left(\frac{\mu^j-c_{\alpha}}{\mu^j} \vee \frac{c_{\alpha}-\mu^j}{1-\mu^j} \right) \leq  \sum_{j=1}^N (a^j+b^j) =1. 
\end{align*}

For the converse, suppose there exists some $c \in [0,1]$ such that 
 $\sum_{j=1}^N \left(\frac{\mu^j-c}{\mu^j} \vee \frac{c-\mu^j}{1-\mu^j} \right) \leq 1$. Denote $s:=\sum_{j=1}^N \left(\frac{\mu^j-c}{\mu^j} \vee \frac{c-\mu^j}{1-\mu^j} \right)$. For each $j=1,\dotso,N$,  if $c \geq \mu^j$, we take  $$a^j=(c-\mu^j)\left(\frac{1}{s}-s\right), \quad b^j=(c-\mu^j)\left(1+\frac{\mu^j}{s(1-\mu^j)}\right),$$
 and if $\mu^j \geq c$, 
 $$a^j=(\mu^j-c) \left(\frac{1-\mu^j}{s\mu^j}+1\right), \quad b^j=(\mu^j-c)\left(\frac{1}{s}-1\right).$$
 It can be easily verified that $a^j, b^j\in [0,1]$, $\sum_{j=1}^N (a^j+b^j)=1$, and $(a^j,b^j)$ satisfies \eqref{def:c}. Therefore, it is a balanced strategies. 
\end{proof}
\begin{proof}[Proof of Lemma \ref{lem:lineargrowth}]
For any $\epsilon>0$, there exists a mollifier $\eta$ such that 
\begin{align*}
| \eta * \Phi- \Phi|_{\infty} < \epsilon.
\end{align*}
Define $\tilde{\Phi}:= \eta * \Phi$, and it suffices for us to show that
\begin{align}\label{eq:lineargrowth}
 |u^M(t, \cdot)- \tilde{\Phi}(\cdot) |_{\infty} \leq C(2-t).
\end{align}
According to the terminal condition of $u^M$, the inequality \eqref{eq:lineargrowth} holds for $t=1$. Assume it is true for $t <1$, we prove for $t-\frac{1}{M}$. Due to the dynamical programming equations and our induction, we get that 
\begin{align*}
|u^M(t- 1/M,x) - \tilde{\Phi}(x) | =& \left|\min_{\phi} \max_{\a} \E^{\phi,\a} \left[u^M(t,x+\Delta X / \sqrt{M} )\right] -\tilde{\Phi}(x) \right| \\
\leq &  \left|\min_{\phi} \max_{\a} \E^{\phi,\a} \left[u^M(t,x+\Delta X / \sqrt{M} )- \tilde{\Phi}(x+\Delta X/\sqrt{M}) \right. \right. \\
& \quad \quad \quad \quad \quad \quad \quad  + \left. \left. \tilde{\Phi}(x+\Delta X/\sqrt{M}) - \tilde{\Phi}(x) \right] \right| \\
\leq & C(2-t) +  \left | \min_{\phi} \max_{\a} \E^{\phi,\a} \left[\tilde{\Phi}(x+\Delta X/\sqrt{M})-\tilde{\Phi}(x) \right] \right|
\end{align*}
Applying Taylor expansion to $\tilde{\Phi}(x)$, we obtain that 
\begin{align*}
\min_{\phi} \max_{\a} \E^{\phi,\a} \left[ \tilde{\Phi}(x+\Delta X/ \sqrt{M})- \tilde{\Phi}(x) \right] =& \frac{1}{\sqrt{M}} \min_{\phi}\max_{\a} \E^{\phi,\a}\left[\nabla \tilde{\Phi}(x) \cdot \Delta X \right]+ O(1/M) \\
=& \frac{1}{\sqrt{M}} \min_{\phi} \max_{\a} \sum_{i=1}^N [\pa_i \tilde{\Phi}- \phi^i ] \E^{\a}[\Delta G_i] +O(1/M). 
\end{align*}
Choosing $\phi_i=\pa_i \tilde{\Phi}$, and $\a \in \mathcal{A}_B$, it can be easily checked that the minimax is zero. Since the second derivative of $\tilde{\Phi}$ is upper bounded, there exists a constant $C>0$ such that $$ \min_{\phi} \max_{\a} \E^{\phi,\a} \left[ \tilde{\Phi}(x+\Delta X/ \sqrt{M})- \tilde{\Phi}(x) \right] \leq C/M,$$ 
and hence 
\begin{align*}
|u^M(t- 1/M,x) - \tilde{\Phi}(x) | \leq C(2-(t-1/M)).
\end{align*}
\end{proof}



\vskip 0.2in
\bibliography{ref}

\end{document}